\documentclass{article}

\usepackage[final]{corl_2021} %

\usepackage{authblk}
\makeatletter
\renewcommand\AB@affilsepx{, \protect\Affilfont}
\makeatother

\usepackage{adjustbox}
\usepackage[ruled, vlined]{algorithm2e}
\usepackage{amsfonts}
\usepackage{amsmath}
\usepackage{amssymb}
\usepackage{amsthm}
\usepackage{array}  %
\usepackage{booktabs}
\usepackage{capt-of}
\usepackage[font=small]{caption}
\usepackage{colortbl}
\usepackage{wrapfig}
\usepackage{graphicx}
\usepackage{multirow}
\usepackage{paralist}
\usepackage{times}
\usepackage{tabularx}
\usepackage{thmtools}  %
\usepackage[table, xcdraw]{xcolor}

\usepackage{pifont}
\newcommand{\cmark}{\ding{51}}
\newcommand{\xmark}{\ding{55}}

\definecolor{flodarkpurple}{rgb}{0.288,0.1196,0.7}

\newcommand{\authorhref}[3][flodarkpurple]{\href{#2}{\color{#1}{#3}}}%

\usepackage{cleveref}

\newcommand{\taskography}{\textsc{Taskography}}
\newcommand{\taskographyapi}{\textsc{Taskography-API}}
\newcommand{\scrub}{\textsc{scrub}}
\newcommand{\scrubbed}{\textsc{scrubbed}}
\newcommand{\seek}{\textsc{seek}}
\newcommand{\sg}{\textsc{3DSG}}
\newcommand{\sgs}{\textsc{3DSG}s}
\newcommand{\graytext}[1]{\textcolor{flodarkpurple}{#1}}

\newtheorem{definition}{Definition}
\newtheorem{domain}{Domain}
\newtheorem{proposition}{Proposition}

\title{\taskography{}: Evaluating robot task planning over large 3D scene graphs \\ \vspace{0.2cm} \large{Project page: \authorhref{https://taskography.github.io}{https://taskography.github.io}} \vspace{-.6cm}}

\author[1]{\authorhref{http://agiachris.github.io/}{Christopher Agia}\thanks{Authors contributed equally. Order determined by academic juniority.}\ \ }
\author[2]{\authorhref{https://krrish94.github.io}{Krishna Murthy Jatavallabhula}$^*$}
\author[1]{\authorhref{https://www.linkedin.com/in/khodeir/?originalSubdomain=ca}{Mohamed Khodeir}}
\author[3]{\authorhref{http://miksik.co.uk/}{Ondrej Miksik}}
\author[3]{\authorhref{http://vibhavvineet.info/}{Vibhav Vineet}}
\author[4]{\authorhref{http://www.mustafamukadam.com/}{Mustafa Mukadam}}
\author[2]{\authorhref{http://liampaull.ca}{Liam Paull}}
\author[1,5]{\authorhref{http://www.cs.toronto.edu/\~florian/}{Florian Shkurti}}
\affil[1]{\href{https://www.utoronto.ca/}{University of Toronto}}
\affil[2]{\href{https://montrealrobotics.ca}{Montreal Robotics and Embodied AI Lab}, \href{https://mila.quebec/en}{Mila}, \href{https://umontreal.ca}{Universit\'e de Montr\'eal}}
\affil[3]{\href{https://www.microsoft.com/en-us/research/lab/mixed-reality-ai-zurich/}{Microsoft}}
\affil[4]{\href{https://mcgill.ca}{Facebook AI Research}}
\affil[5]{\href{https://vectorinstitute.ai/}{Vector Institute}\vspace{-1cm}}

\begin{document}
\maketitle

\begin{abstract}
    3D scene graphs (\sgs{}) are an emerging description; unifying symbolic, topological, and metric scene representations.  %
    However, typical \sgs{} contain hundreds of objects and symbols even for small environments; rendering task planning on the \emph{full} graph impractical.
    We construct \textbf{\taskography{}}, the first large-scale robotic task planning benchmark over \sgs{}.
    While most benchmarking efforts in this area focus on \emph{vision-based planning}, we systematically study \emph{symbolic} planning, to decouple planning performance from visual representation learning.
    We observe that, among existing methods, neither classical nor learning-based planners are capable of real-time planning over \emph{full} \sgs{}.
    Enabling real-time planning demands progress on \emph{both} (a) sparsifying \sgs{} for tractable planning and (b) designing planners that better exploit \sg{} hierarchies.
    Towards the former goal, we propose \scrub{}, a task-conditioned \sg{} sparsification method; enabling classical planners to match and in some cases surpass state-of-the-art learning-based planners.
    Towards the latter goal, we propose \seek{}, a procedure enabling learning-based planners to exploit \sg{} structure, reducing the number of replanning queries required by current best approaches by an order of magnitude.
    We will open-source all code and baselines to spur further research along the intersections of robot task planning, learning and \sgs{}.

\end{abstract}

\keywords{Robot task planning, 3D scene graphs, learning to plan, benchmarks}

\vspace{-.15cm}
\section{Introduction}
\label{sec:introduction}
\vspace{-.15cm}

Real-world robotic task planning problems in large environments require reasoning over tens of thousands of object-action pairs.
Faced with long-horizon tasks and an abundance of choices, state-of-the-art task planners struggle with an efficiency-reliability trade-off in grounding actions towards the goal.
Hence, designing actionable scene abstractions suitable for a range of robotic tasks has drawn long-standing attention from the robotics and computer vision communities~\cite{spatial-semantic-hierarchy, towards-topological-maps, multi-hierarchical-maps, conceptual-spatial-repr, multiversal-semantic-maps, kimera}. 

A promising approach for building symbolic abstractions from raw perception data are 3D scene graphs (\sgs{}, see Fig.~\ref{fig:scenegraph})~\cite{3dscenegraph, 3ddynamicscenegraphs, 3dscenegraph-cybernetics} -- hierarchical representations of a scene that capture metric, semantic, and relational information, such as affordances, properties, and relationships among scene entities.
While \sgs{} have to date been applied to simpler planning problems like goal-directed navigation~\cite{kimera, ravichandran2021hierarchical}, active object search~\cite{hms}, and node classification~\cite{neuraltrees}, their amenability to more complex robotic task planning problems has yet to be thoroughly evaluated.

To investigate the joint application of \sgs{} and modern task planners to complex robotics tasks we propose \textbf{\taskography{}}: the first large-scale benchmark comprising a number of challenging task planning domains designed for \sgs{}.
Analyzing planning times and costs on a diversity of domains in \taskography{} reveals that neither classical nor learning-based planners are capable of real-time planning over full \sgs{}, however, that they become so only when \sgs{} are sparsified.

Many real-world problems only require reasoning over a small subset of scene objects.
E.g., the task ``\emph{fetch a mug from the kitchen}'' primarily involves reasoning about scene elements associated with mugs or kitchens, rendering a vast majority of the remaining environment contextually irrelevant.
Most planners aren't able to exploit such implicitly defined task contexts, instead spending most of their computation time reasoning about extraneous scene attributes and actions~\cite{ploi} (see Fig.~\ref{fig:scrub-domain-stats}).

We argue that performant task planning over \sgs{} demands progress on two fronts: (a) sparsifying \sgs{} to make planning problems tractable, and (b) designing task planners that exploit the spatial hierarchies encapsulated in \sgs{}.
To address (a), we present \textbf{\scrub{}}--a planner-agnostic strategy guaranteed to produce a minimal \emph{sufficient} object set for grounded planning problems.
That is, planning on this reduced subset of scene entities suffices to solve the planning problem defined over the full \sg{}.
Classical planning over state spaces (\sgs{}) augmented by \scrub{} outperforms state-of-the-art learning-based planners on the majority of tasks on our benchmark, without requiring any prior learning, establishing a strong baseline for future work in robotic task planning.
To address (b), we present \textbf{\seek{}}: a procedure tailored to \sgs{}, which supplements learning-based incremental planners by imposing \sg{} structure, ensuring all objects in the \emph{sufficient} set are reachable from the start state.
In our experiments, augmenting state-of-the-art planners with \seek{} results in computational savings and an order of magnitude fewer replanning iterations.

In summary, we make the following contributions:
\begin{compactitem}
\item \taskography{}: a large-scale benchmark to evaluate robotic task planning over \sgs{},
\item \scrub{}: a planner-agnostic strategy to adapt \sgs{} for performant planning,
\item \seek{}: a procedure that enables learning-based planners to better exploit \sgs{}
\end{compactitem}

We will open-source all code and baselines in \taskographyapi{}, enabling the construction of new task planning domains, and benchmarking the performance of newer learning-based planners. %

\vspace{-.5cm}
\section{Related work}
\label{sec:relatedwork}

Early research in \textbf{symbolic planning} was centered around \emph{optimal} planning~\cite{kautz1996pushing, kautz1999unifying, satplan, delfi, mcts-survey}; planners producing solutions that preserve cost or plan length optimality.
These methods are computationally expensive and thereby untenable to even moderately sized problems.
This spurred work on \emph{satisficing} planners that forgo optimal solutions for cheaper, feasible plans.
Notable paradigms include regression planning~\cite{bonet1999planning}, tree search~\cite{kaindl1990tree}, and heuristic search~\cite{hsp, ff, fd, cerberus, decstar}. 
Whilst the many successes of heuristic planners~\cite{ipc3, ipc2014}, computing low-cost informative heuristics is deterred by many extraneous objects~\cite{chitnis2020camps, ploi}; an inauspicious characteristic of large \sgs{}.

\textbf{Robot task planning} techniques have focused on constructing more effective representations to plan upon~\cite{galindo2004improving, prm, task-planning-semantic-maps}.
There are also approaches that integrate task and motion planning~\cite{kaelbling2013integrated, tamp-scorespace, garrett2021integrated}--further demonstrated in hierarchical task space~\cite{tamp-hierarchical}--but which fall outside the scope of our work. 
Several approaches exploit task hierarchies for robot task planning~\cite{stock2015online} and control~\cite{gonzalez2017three, beetz2012cognition, learning-skill-trees}. 
Different from these, our work focuses on exploiting abstractions in \emph{spatial structure} encapsulated in \sgs{}, not to be conflated with hierarchical planners that exploit \emph{task structure}~\cite{bercher2019survey}.

State-of-the-art \textbf{learning-based} planners have demonstrated promising performance in small-to-moderate problem sizes. However, techniques such as relational policy learning~\cite{relationalpolicy}, relational heuristic learning~\cite{relational-heuristic}, action grounding~\cite{gnad2019learning}, program guided symbolic planning~\cite{oh2017zero,andreas2017modular,denil2017programmable,sun2019program,yang2021program}, and regression planning networks~\cite{regression-planning-networks} fail in large problem instances with branching factors and operators of the order considered (see Fig. \ref{fig:taskography-benchmark}) in the \taskography{} benchmark. 
Moreover, several planners that learn to search~\cite{mcts-nets, qmdp-net, karkus2018integrating, guez2019investigation} depend on hard-to-obtain dense rewards or do not scale with domain complexity~\cite{neuraltaskgraphs,neuraltaskprogramming,mwalk}.

The simplification of planning problems via \textbf{pruning strategies} to enable efficient search has been explored in both propositional~\cite{fivser2020lifted, fivser2020strengthening, gnad2019learning} and numeric~\cite{task-scoping} planning contexts. 
Among these, PLOI~\cite{ploi} is a particularly performant learning-based approach that leverages object-centric relational reasoning~\cite{battaglia2018relational, kipf2018neural, kipf2019contrastive, locatello2020object} to score and prune \emph{extraneous objects} to the task.
While PLOI outperforms existing classical planners on the \taskography{} benchmark, it incurs a large number of replanning steps owing to inaccurate neural network predictions; and inability to exploit \sg{} hierarchies.
Our proposed \seek{} procedure decreases replanning steps by two orders of magnitude.

\textbf{Planning benchmarks} in the symbolic planning communities have featured a variety of tasks with time complexities ranging from polynomial (e.g., shortest-path) to NP-hard problems (e.g., traveling salesman). 
There also exists a handful of environments~\cite{ai2thor, alfred, puig2018virtualhome, habitat, weihs2021visual} for benchmarking learned action policies from language directives and ego-centric visual observations, task and motion planning~\cite{gan2021threedworld}, or the modelling of physical interactions~\cite{igibson, sapien}.
Another recent benchmark~\cite{robot-task-planning-benchmark-2019} only supports navigation and block-stacking tasks.
However, there isn't currently a large-scale benchmark tailored to robotic task planning in \sgs{} with several hundreds of objects.

\section{Background}
\label{sec:background}

\textbf{Task planning.} A task planning problem $\Pi$ is a tuple $\langle \cal{O}, \cal{P}, \cal{A}, \cal{T}, \cal{C}, \cal{I}, \cal{G} \rangle$. As a running example, consider the task \graytext{find an apple, slice it, and place it on the counter}. $\cal{O}$ is the set of all ground objects (instances) in the problem. %
$\cal{P}$ is a set of properties, each defined over one or more objects; \graytext{weight(apple) = 70 grams}. \textbf{Predicates} are subclasses of properties in that they are boolean-valued; \graytext{canPlace(apple, refrigerator) = True}. %
$\cal{A}$ is a finite set of lifted actions operating over object tuples; \graytext{slice(apple)}, \graytext{place(apple, counter)}.
$\cal{T}$ is a transition model and $\cal{C}$ denotes state transition costs.
$\cal{I}$ and $\cal{G}$ are initial and goal states. 
A state is an assignment of values to all possible properties grounded over objects. 
For the running example, a goal state may be specified as \graytext{on(apple, counter)=True and sliced(apple)=True}.
Planning problems may be grounded--\graytext{slice \emph{this} apple}, or lifted--\graytext{slice \emph{an} apple}.

\begin{wrapfigure}{R}{.5\linewidth}
\vspace{-.5cm}
\centering
\includegraphics[width=\linewidth]{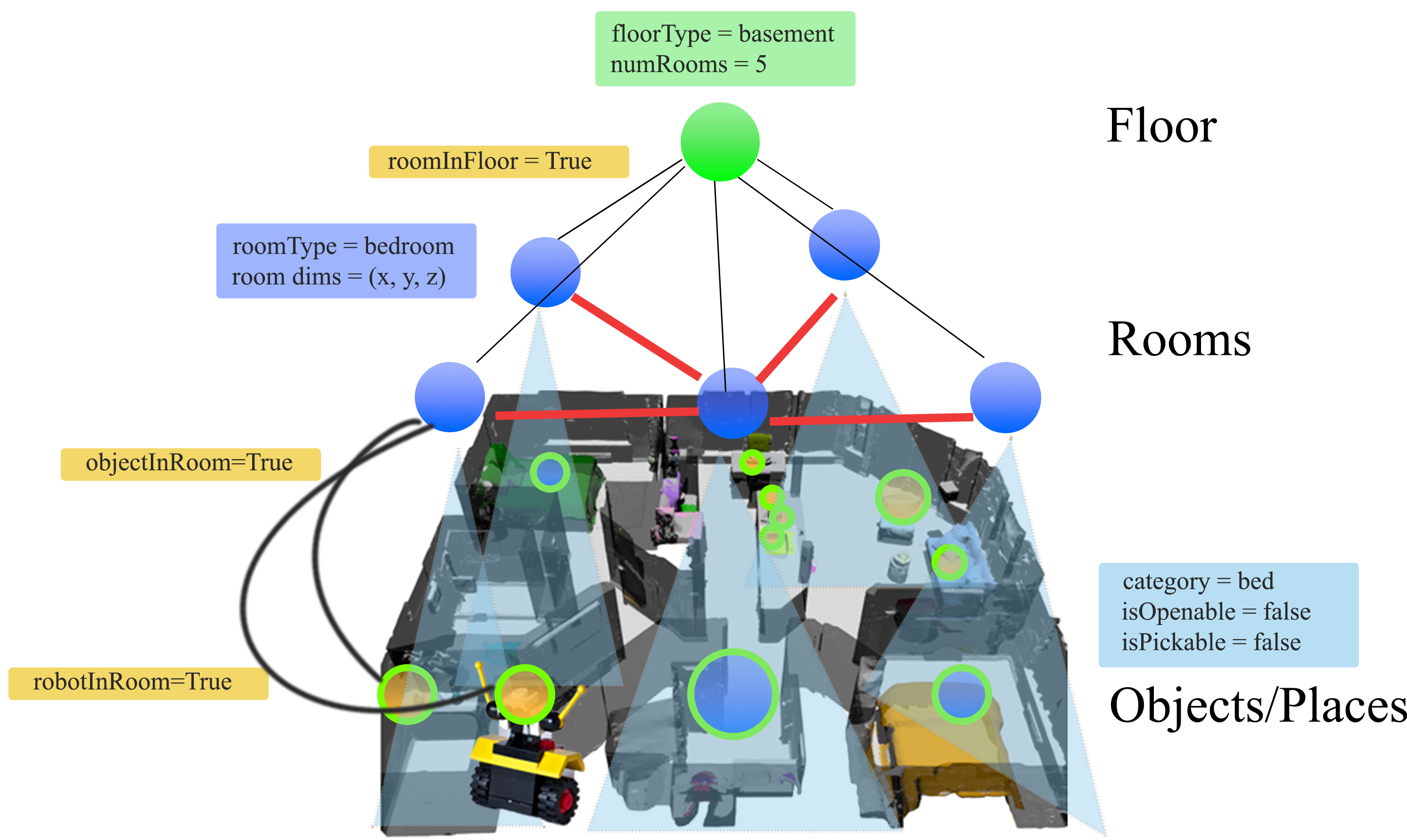}
\caption{\small A \emph{state} in a planning problem specified over a \sg{}. Nodes are scene entities and store unary predicates. Edges indicate binary predicates (relations). A goal is a conjunction of unary and binary literals. We only show a subset of relations for brevity. E.g., if the robot executes an action that moves it to another room, the \graytext{robotInRoom} relation shown in this figure will be set to \graytext{False} for the room on the lower left.}
\label{fig:scenegraph}
\vspace{-.3cm}
\end{wrapfigure}

\textbf{3D scene graphs (\sgs{}).} A \sg{}~\cite{3dscenegraph, 3ddynamicscenegraphs} is a hierarchical multigraph $G = (V, E)$ with $k \in \{1 \cdots K\}$ levels, where $V^k \in V$ denotes the set of vertices at level $k$.
Edges originating from a vertex $v \in V^k$ may only terminate in $V^{k-1} \cup V^k \cup V^{k+1}$ (i.e., edges connect nodes within one level of each other).
Each \sg{} in our work comprises at least 5 levels with increasing spatial precision as we move down the hierarchy: the topmost level in the hierarchy is a root node representing a scene. This node branches out to indicate the various \emph{floors} in the building, which in turn branches out to denote various \emph{rooms} in a floor, and subsequently \emph{places} within a room.
A place is a collection of \emph{objects}, which may themselves contain other \emph{objects} (to allow for container types such as cabinets and refrigerators).\footnote{The lowest level in~\cite{3ddynamicscenegraphs} is a metric-semantic mesh. However since our focus is on symbolic planning, we only require scene graph levels that contain \emph{objects}.}
At each level, edges indicate various types of relations among nodes (e.g., at the room level, an edge indicates the existence of a traversible path between two rooms; at the object level, edges indicate multiple affordance relations).
Each node also stores semantic attributes such as node type, functionality, affordances, etc., following~\cite{3dscenegraph}.

\section{\taskography{}}
\label{sec:taskography}

We propose \taskography{}: the first large scale benchmark to evaluate symbolic planning over \sgs{}. Currently, \taskography{} comprises 20 challenging robotic task planning domains totaling 3734 tasks. %
Different from current benchmarks for embodied AI that focus primarily on egocentric \emph{visual} reasoning~\cite{alfred, shridhar2020alfworld, gan2020threedworld, ai2thor, sapien, habitat}; \taskography{} is designed to evaluate \emph{symbolic} reasoning over \sgs{}.
To emulate the complexity of real-world task planning problems, \taskography{} builds atop the Gibson~\cite{gibson} dataset comprising real-world scans of large building interiors (averaging 2-3 floors per building; 7 rooms per floor), and their corresponding \sgs{}~\cite{3dscenegraph}.

\textbf{Augmenting \sgs{} with plannable attributes.} A prerequisite for planning over \sgs{}---absent in existing work~\cite{3dscenegraph-cybernetics, 3dscenegraph, 3ddynamicscenegraphs}---is  a database of \emph{plannable attributes}: predicates, actions, and transition models.
To support task planning, we augment each \sg{} in Gibson~\cite{gibson} (tiny and medium splits) with several layers of additional unary and binary predicates.
For each \sg{} node, we obtain class labels, object dimensions and pose from~\cite{3dscenegraph}. 
We annotate object affordances by building a knowledge base of lifted object-action pairs and recursively applying it to every \sg{} node, while accounting for exceptions (objects that are concealed or contained within others).
We also detect \emph{door} objects in the \sg{} and use this to add additional edges describing room connectivity.
We annotate objects with all possible properties in our planning domains (e.g., ``\emph{is this object typically a receptacle?}''). Our rich property set (\emph{plannable attributes}) is chosen to support a wide range of realistic-robotic tasks geared towards large (building-scale) \sgs{}.

\begin{figure}[!htb]
    \centering
    \begin{minipage}{.5\textwidth}
        \centering
        \includegraphics[width=.85\linewidth]{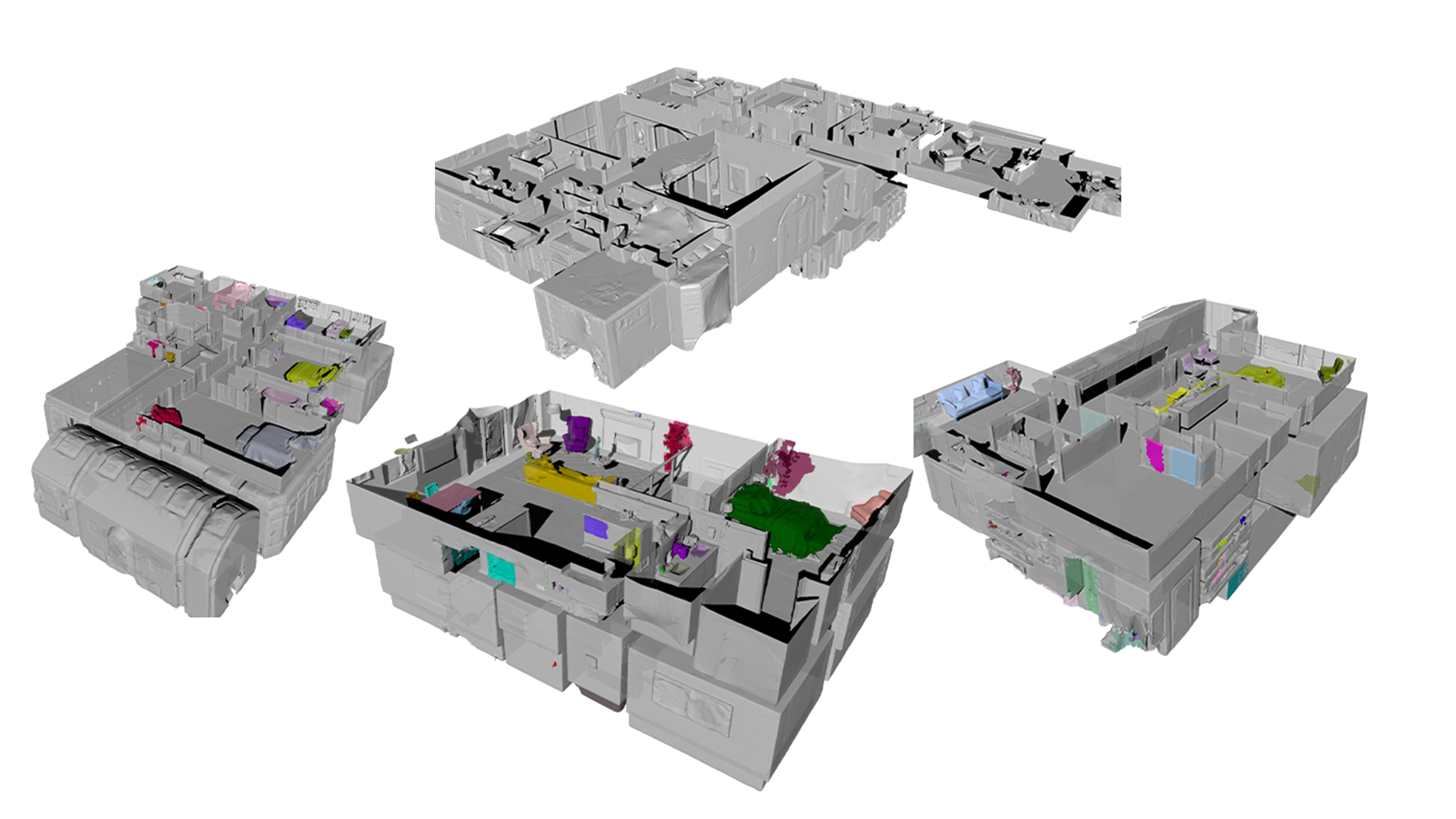}
    \end{minipage}
    \begin{minipage}{.245\textwidth}
        \centering
        \includegraphics[width=\linewidth]{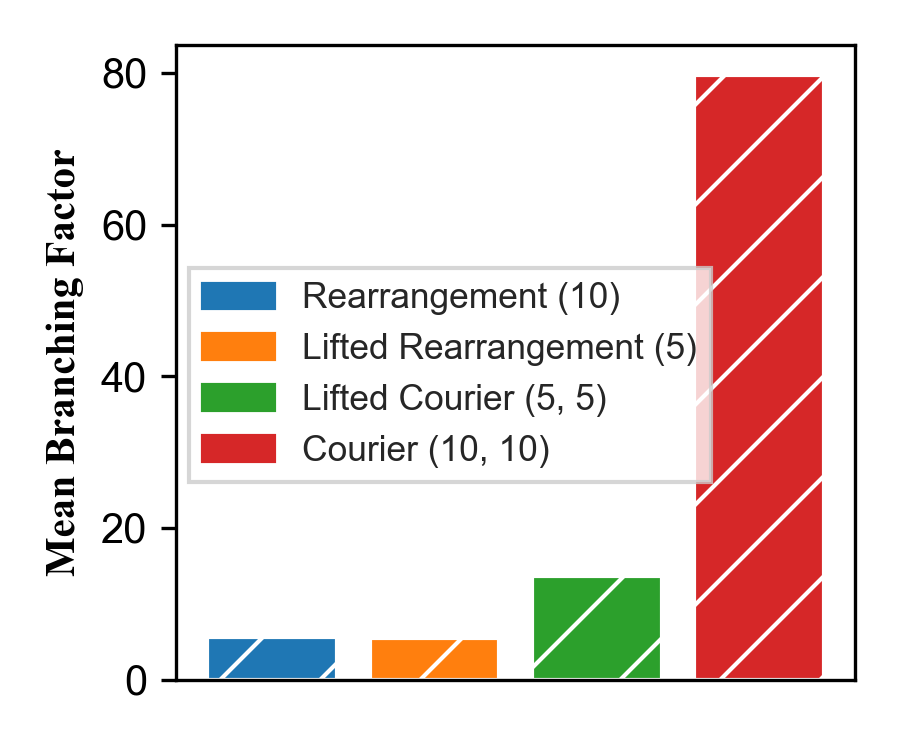}
    \end{minipage}
    \begin{minipage}{.24\textwidth}
        \centering
        \includegraphics[width=\linewidth]{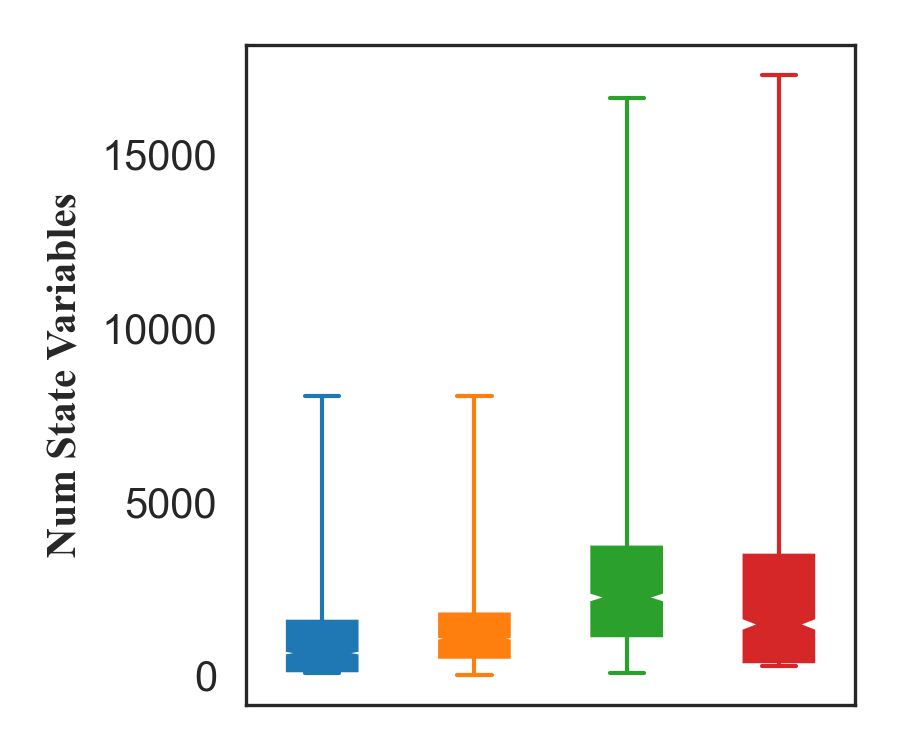}
    \end{minipage}
    \caption{The \textbf{\taskography{} benchmark} comprises large-scale planning problems defined over buildings from the Gibson dataset~\cite{gibson}. (Left) Representative buildings from Gibson~\cite{gibson}. (Middle/Right) We feature a variety of problem classes ranging in scale and complexity as illustrated by the domain statistics.}
    \label{fig:taskography-benchmark}
\end{figure}

\textbf{Benchmark statistics.}
Each of the 20 \taskography{} domains specifies a class of planning problems that resemble real-world use cases (and theoretically complex extensions) that a robot would encounter in office, house, or building scale environments.
These domains range from grounded planning domains to lifted planning domains, domains with no extraneous objects to domains where most objects are extraneous, and domains for which polynomial time solutions exist to NP-hard problems. The simplest domains in the benchmark have 1000 state variables and an average branching factor of 5; for hard domains, these are 4000 and 60 respectively (see Fig.~\ref{fig:taskography-benchmark}).

\textbf{\taskographyapi{}.} Our project page (\authorhref{https://taskography.github.io}{https://taskography.github.io}) will host code and data used in this work. In \taskographyapi{}, an open-source python package, we provide access to 18 classical and learning-based symbolic planners, templates to implement novel domains, and methods to generate problem instances of varying complexities and train/evaluate learning-based planners.

\textbf{Planners considered.}
\taskography{} supports a comprehensive set of planners to facilitate standardized evaluation on novel domains. The following planners are available at the time of writing.
\begin{compactitem}
\item \textbf{Optimal planners}: Fast Downward (FD) with the \texttt{opt-lmcut} heuristic~\cite{fd}, SatPlan~\cite{satplan}, Delfi~\cite{delfi}, DecStar-optimal~\cite{decstar}, and Monte Carlo tree search.
\item \textbf{Satisficing planners}: Fast Forward (FF), FF with axioms (FF-X)~\cite{ff}, Fast Downward (FD) with the \texttt{lama-cut} heuristic~\cite{fd}, DecStar-satisficing~\cite{decstar}, Cerberus~\cite{cerberus}, Best First Width Search (BFWS) \cite{bfws}, and regression planning.
\item \textbf{Learning-based planners}: Relational policy learning~\cite{relationalpolicy}, Planning with learned object importance (PLOI)~\cite{ploi} (and variants -- see Sec.~\ref{sec:seek}).
\end{compactitem}

\textbf{General assumptions.} To facilitate evaluation of all of these classes of planners, the first edition of our benchmark only considers \emph{fully observable} tasks and \emph{discrete} state and action spaces.
All goal states are specified as \emph{conjunctions} of literals.
While we make no distinction between deterministic or stochastic transitions, all current experiments assume a \emph{closed world}, i.e., all possible lifted actions and effects are known apriori.

\subsection{Robot planning domains: Case studies}

The full \taskography{} benchmark comprises 20 domains.
We discuss the four task categories from which all domains are constructed that we believe to be interesting to a broad robotics audience.

\begin{domain}
\emph{Rearrangement($k$)}: Based on the recently proposed rearrangement challenge~\cite{rearrangement}, this task requires a robot randomly spawned to rearrange a set of $k$ objects of interest into $k$ corresponding receptacles. The robot often needs to execute multiple other actions along the way, such as opening/closing doors, navigating to goals, planning the sequence of objects to visit, etc.
\end{domain}
\begin{domain}
\emph{Courier ($n, k$)}: A robot that couriers objects is equipped with a knapsack of maximum payload capacity of $n$ units. The robot needs to locate and courier $k$ objects (of varying weights $w \in \{1, 2, 3\}$ units) to $k$ distinct delivery points.
The knapsack can be used to stow and retrieve items in random-access fashion; effectively embedding a combinatorial optimization problem into the task.
Stow and retrieve actions increase branching, necessitating far deeper searches.
\end{domain}
\vspace{-\parskip}
We also provide \emph{lifted} variants of these tasks. Here, goals are specified over desired object-receptacle class relations (e.g., ``put a cup on a table'') as opposed to over object instances (e.g., ``put this cup on the table''). 
These tasks introduce ambiguity in both the search of classical task-planners and learning-based techniques, which must now distinguish object instances of relevant classes.
\begin{domain}
\emph{Lifted Rearrangement ($k$)}: A lifted version of the rearrangement domain where the goals are specified at an object category level, as opposed to an instance level.
\end{domain}
\begin{domain}
\emph{Lifted Courier ($n$, $k$)}: A lifted version of the courier domain where the goals are specified at an object category level, as opposed to an instance level.
\end{domain}
\vspace{-\parskip}
To promote compatability with a range of planning systems~\cite{ipc2014, planutils}, we represent all tasks in PDDL format~\cite{pddl, pddlgym}. We also include mechanisms for translating tasks into alternative problem definition languages that are essential for some of our supported planners~\cite{satplan}.

\newcommand{\STAB}[1]{\begin{tabular}{@{}c@{}}#1\end{tabular}}

\begin{table}[]
    \centering
    \caption{\textbf{\textsc{Taskography}} benchmark results on select grounded and lifted \emph{Rearrangement} (\textbf{Rearr}) and \emph{Courier} (\textbf{Cour}) \sg{} domains. Planning times are reported in seconds and do not incorporate planner-specific domain translation times (factored into planning timeouts). A `-' indicates planning timeouts or failures (10 minutes for optimal planners, 30 seconds for all others). Results are aggregated over 10 random seeds.
    Optimal task planning is infeasible in larger problem instances or for more complex domains, while most satisficing planners are unable to achieve real-time performance. PLOI~\cite{ploi}, a recent learning-based planner consistently performs the best across all domains.}
    \adjustbox{max width=\linewidth}{
    \begin{tabular}{@{}llrrrrrrrrrrrrrrrrrr@{}}
    \toprule
    &  & \multicolumn{3}{c}{\textbf{Rearr(1)} Tiny} & \multicolumn{3}{c}{\textbf{Rearr(2)} Tiny} & \multicolumn{3}{c}{\textbf{Rearr(10)} Medium} & \multicolumn{3}{c}{\textbf{Cour(7, 10)} Medium} & \multicolumn{3}{c}{\textbf{Lifted Rearr(5)} Tiny} & \multicolumn{3}{c}{\textbf{Lifted Cour(5, 5)} Tiny} \\
    \cmidrule{3-5} \cmidrule{6-8} \cmidrule{9-11} \cmidrule{12-14} \cmidrule{15-18} \cmidrule{18-20}
     & \textbf{Planner}          &    \textbf{Len.} &    \textbf{Time} &   \textbf{Fail} &    \textbf{Len.} &    \textbf{Time} &   \textbf{Fail} &    \textbf{Len.} &    \textbf{Time} &   \textbf{Fail} &    \textbf{Len.} &    \textbf{Time} &   \textbf{Fail} &    \textbf{Len.} &    \textbf{Time} &   \textbf{Fail} &    \textbf{Len.} &    \textbf{Time} &   \textbf{Fail} \\
    \midrule
    \multirow{5}{*}{\STAB{\rotatebox[origin=c]{90}{optimal}}} & \textbf{FD-seq-opt-lmcut} &  15.77 &  24.81 &   0.04 &  \textbf{25.80}  & 104.47 &   0.55 &   -    &   -    &   1.00    &   -    &   -    &   1.00    &   -    &   -    &   1.00    &   -    &   -    &   1.00    \\
    & \textbf{SatPlan}          &  14.77 &  10.35 &   0.45 &  \underline{26.67} &   3.27 &   0.67 &   -    &   -    &   1.00    &   -    &   -    &   1.00    &   -    &   -    &   1.00    &   -    &   -    &   1.00    \\
    & \textbf{Delfi}            &  15.13 &   0.36 &   0.16 &  29.10  &  27.77 &   0.29 &   -    &   -    &   1.00    &   -    &   -    &   1.00    &   -    &   -    &   1.00    &   -    &   -    &   1.00    \\
    & \textbf{DecStar-opt-fb}   & -    &   -    &   1.00    & -    &   -    &   1.00    &   -    &   -    &   1.00    &   -    &   -    &   1.00    &   -    &   -    &   1.00    &   -    &   -    &   1.00    \\
    & \textbf{MCTS} & - & - & 1.00 & - & - & 1.00 & - & - & 1.00 & - & - & 1.00 & - & - & 1.00 & - & - & 1.00 \\
    \midrule
    \multirow{6}{*}{\STAB{\rotatebox[origin=c]{90}{satisficing}}} & \textbf{FF}               &  16.71 &   \underline{0.19} &   \textbf{0.00}    &  34.44 &   0.55 &   \textbf{0.00}    & \underline{159.04} &   5.30  &   0.09 & 128.41 &   6.62 &   0.24 &  62.86 &   3.40  &   0.47 &  \textbf{57.74} &   4.03 &   0.44 \\
    & \textbf{FF-X}             &  16.71 &   0.25 &   \textbf{0.00}    &  34.44 &   0.58 &   \textbf{0.00}    & 159.80  &   5.02 &   \underline{0.08} & \underline{128.19} &   6.72 &   0.24 &  67.88 &   3.48 &   0.89 &  61.19 &   7.56 &   0.77 \\
    & \textbf{FD-lama-first}    &  15.19 &   2.96 &   0.33 &  38.47 &   3.25 &   0.18 & 208.28 &   6.35 &   0.49 & 156.34 &   4.92 &   0.29 &  66.81 &   3.20  &   0.49 &  61.13 &   3.34 &   0.56 \\
    & \textbf{Cerberus-sat}     &  \textbf{11.50}  &  12.00    &   0.85 & -    &   -    &   1.00    & -    &   -    &   1.00    & -    &   -    &   1.00    & -    &   -    &   1.00    & -    &   -    &   1.00    \\
    & \textbf{Cerberus-agl}     &  14.77 &   5.13 &   0.45 &  33.00    &   7.30  &   0.49 & 176.60  &   8.91 &   0.72 & \textbf{125.73} &  12.99 &   0.83 &  \underline{60.50}  &   7.62 &   0.60  &  59.19 &   7.05 &   0.77 \\
    & \textbf{DecStar-agl-fb}   &  \underline{14.72} &   2.62 &   0.55 &  34.96 &   2.58 &   0.58 & 211.16 &   7.20  &   0.82 & 132.60  &   4.50  &   0.58 &  66.30  &   3.02 &   0.71 &  \underline{58.75} &   4.46 &   0.71 \\
    & \textbf{BFWS}             &  15.56 &   0.90  &   0.22 &  32.16 &   \underline{0.37} &   0.18 & \textbf{151.17} &   \underline{0.41} &   0.23 & 152.71 &   \underline{1.13} &   \underline{0.21} &  \textbf{56.90}  &   \underline{0.94} &   \underline{0.41} &  61.92 &   \underline{2.30}  &   \underline{0.43} \\
    & \textbf{Regression-plan} & - & - & 1.00 & - & - & 1.00 & - & - & 1.00 & - & - & 1.00 & - & - & 1.00 & - & - & 1.00 \\
    \midrule
    \multirow{2}{*}{\STAB{\rotatebox[origin=c]{90}{learn}}} & \textbf{Relational policy~\cite{relationalpolicy}} & - & - & 1.00 & - & - & 1.00 & - & - & 1.00 & - & - & 1.00 & - & - & 1.00 & - & - & 1.00 \\
    & \textbf{PLOI~\cite{ploi}}             &  16.45 &   \textbf{0.00$^*$}    &   \textbf{0.00}    &  37.04 &   \textbf{0.00$^*$}    &   \textbf{0.00}    & 213.43 &   \textbf{0.17} &   \textbf{0.00}    & 161.90  &   \textbf{0.34} &   \textbf{0.00}    &  78.68 &   \textbf{0.22} &   \textbf{0.24} &  71.71 &   \textbf{0.26} &   \textbf{0.26} \\
    \bottomrule
    \end{tabular}
    } %
    \label{table:taskography-main}
\end{table}

\subsection{Benchmarking classical and learned planners on \taskography{}}
\label{sec:taskography-results-discussion}

We present the empirical results on the \taskography{} benchmark across several classes of task-planners in Table. \ref{table:taskography-main}. (Please consult supplementary material for a number of additional results).

\textbf{Evaluation protocol.} We treat the evaluation of optimal planners separately to the remaining methods.
Optimal planners are not intended to be fast unlike satisficing and learning-based variants. Rather, they compute a solution of minimum length (not necessarily unique) to a given problem.
Optimal planners are hence allotted 10 minutes to solve each problem, while satisficing and learning-based planners are allotted 30 seconds.
For learning-based methods, we evaluate results over 10 random seeds for statistical significance. We report standard deviations in the supplementary material.
All domains comprise 40 training problems. The domains tagged \emph{Tiny} and \emph{Medium} comprise 55 and 182 test problems respectively, unless otherwise specified.

\textbf{Optimal planners work only on the simplest of domains.} 
Despite the reasonable performance of optimal planners on the \emph{Rearrangment(1)} domain, they are unable to efficiently scale with increasing task complexity and fail to solve a single task on the Rearrangment (k) and Courier (n, k) domains for $k>2$. 
In particular, the Rearrangement(1) domain is a superset of the grounded hierarchical path planning (HPP) task as described by \citet{kimera}. Because the HPP task does not consider state changes to the scene graph (i.e., directly equating the \sg{} to the planning graph for search), efficient shortest path planning is tractable. However, increasingly complex robot tasks requires more than the mere ability to path plan in \sgs{}.

\begin{wraptable}{r}{.4\textwidth}
    \vspace{-.7cm}
    \caption{Interestingly, task complexity does not correlate strongly with scene complexity. It is instead determined by the number of operators, and avg. branch factor.}
    \begin{minipage}{.4\textwidth}
    \adjustbox{max width=\textwidth}{
\begin{tabular}{lrrrrrr}
\toprule
 & \multicolumn{3}{c}{\textbf{Rearr(10)} Tiny} & \multicolumn{3}{c}{\textbf{Rearr(10)} Medium} \\
 \cmidrule{2-4} \cmidrule{5-7}
 \textbf{Planner}        &   \textbf{Len.} &   \textbf{Time} &   \textbf{Fail} &   \textbf{Len.} &   \textbf{Time} &   \textbf{Fail} \\
\midrule
 \textbf{FF}             & \underline{162.61} &   7.04 &   \textbf{0.07} & \underline{159.04} &   \underline{5.30}  &   \textbf{0.09} \\
 \textbf{FD (satisficing)}  & 205.89 &   7.68 &   0.51 & 208.28 &   6.35 &   0.49 \\
 \textbf{DecStar-agl-fb} & 193.00    &   \underline{6.78} &   0.85 & 211.16 &   7.20  &   0.82 \\
 \textbf{BFWS}           & \textbf{160.93} &   \textbf{0.57} &   \underline{0.18} & \textbf{151.17} &   \textbf{0.41} &   \underline{0.23} \\
\bottomrule
\end{tabular}
} %

    \label{table:domain-vs-scene-complexity}
    \end{minipage}
    \vspace{-1cm}
\end{wraptable}
\textbf{Planning performance degrades with domain complexity, not scene complexity.}
We observe an increase in the number of planning failures and timeouts as satisficing planners are applied to larger \emph{Rearrangement(k)} domains (Table~\ref{table:domain-vs-scene-complexity}). 
Interestingly, larger scenes do not appear to directly correlate with task complexity, as the performance metrics remain largely consistent between the tiny and medium splits of the same domain (Table.~\ref{table:domain-vs-scene-complexity}).

\textbf{Satisficing planners fail in domains requiring long-horizon reasoning.}
In the \emph{Courier(n, k)} domains, satisficing planners tend to produce shorter length solutions by leveraging the knapsack's capacity to stow objects on the way to various delivery points.
However, the planners often display shortsighted behaviours by stowing objects early in the search, depleting knapsack slots that could potentially help further along the task. This yields dead-end configurations and excessive backtracking, and thus, an increase in timeouts is observed.

\textbf{Planners that do not exploit forward heuristics fail due to large branching factors.}
Due to the large branching factor of our domains, common strategies such as Monte-Carlo Tree Search (MCTS) and MC Regression Planning are unable to solve any task within a reasonable time constraint.
For instance, a \emph{Rearrangement(10)} task has an average branching factor of 6.5 for MCTS. Since a reward is only obtained at the end (typical planners take 200 steps to get there), MCTS degenerates to a slow breadth-first search.

\begin{wrapfigure}{r}{.4\textwidth}
    \vspace{-.5cm}
    \begin{minipage}{.4\textwidth}
        \includegraphics[width=\textwidth]{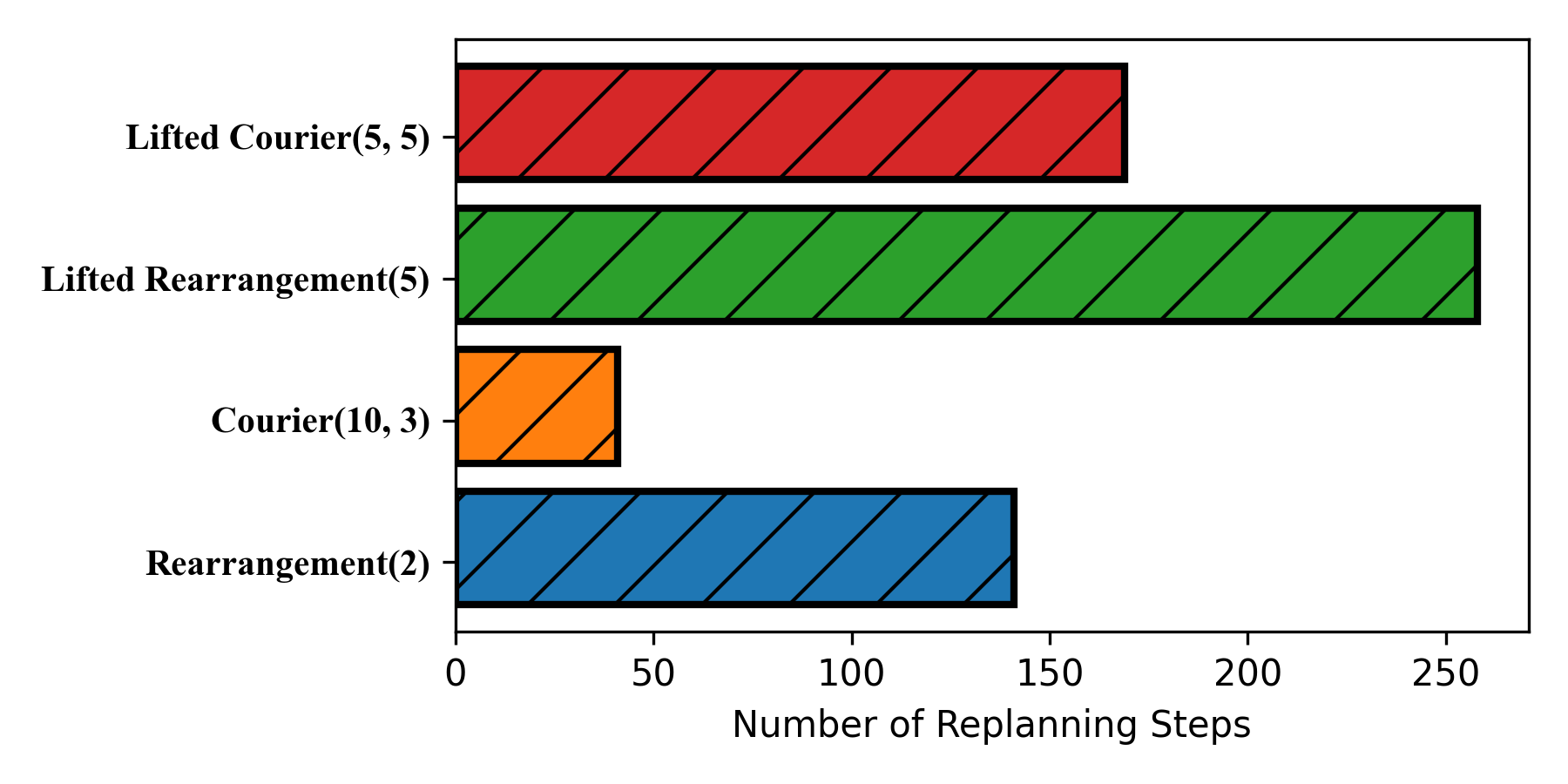}
        \captionof{figure}{Learning-based planners like PLOI outperform all other planners on the benchmark, but still incur significant overhead (number of replanning steps).}
        \label{fig:ploi-replanning-steps}
    \end{minipage}
    \vspace{-.5cm}
\end{wrapfigure}
\textbf{Learning based planners that prune the state space excel on all domains.}
We also evaluate two prevalent learning-to-plan methods based on generalized relational planning~\cite{relationalpolicy} and planning with learned abstractions~\cite{ploi}. 
While the relational policy stuggles to generalize in our domains (long-horizon, sparse rewards), PLOI demonstrates an impressive ability to detect and prune contextually irrelvant parts of the \sgs{}. 
However, it also requires a significant number of replanning steps (see figure to the right) as it often retains objects within a graph without ensuring that all properties and ancestors required to access the object are also preserved.

\textbf{Discussion.}
Our evaluation of existing performant planners on the \taskography{} benchmark consistently reveals two important trends across all domains.
\begin{compactitem}
\item Pruning a \sg{} is essential for real-time performance, more so on challenging domains.
\item While learning-based planners excel across all domains, they require a large number of replanning steps.
\end{compactitem}
These imply that efficient utilization of \sgs{} in real-time robotic task planning requires \emph{both} adapting \sgs{} to better suit existing planners, and enabling performant (learning-based) planners to better exploit \sg{} hierarchies. The remainder of our work addresses these issues.

\vspace{-.3cm}
\section{\scrub{}: Principled sparsification of \sgs{} for efficient planning}
\label{sec:scrub}
\vspace{-.3cm}

\begin{wrapfigure}{R}{.4\linewidth}
\vspace{-.5cm}
    \begin{minipage}{\linewidth}
        \begin{algorithm}[H]
        \tiny
        \SetAlgoLined
        \KwIn{\sg{} $G$, Planning problem $\Pi = \langle \mathcal{O}, \mathcal{P}, \mathcal{A}, \mathcal{T}, \mathcal{C}, \mathcal{I}, \mathcal{G} \rangle$}
        \KwResult{Sparsified \sg{} $\hat{G}$}
        $\hat{\mathcal{O}} = \{ \texttt{} \}$ \tcc*[r]{Init. sufficient object set}
        $g = $ \textsc{Objects}($\mathcal{G}.\textnormal{literals}) \cup \{robot\}$ \tcc*[r]{Init. set of objects in the goal literal set}
        \While{not empty $g$}{
            $\hat{\mathcal{O}} \gets \hat{\mathcal{O}} \cup g$ \\
            $p \gets $ all binary predicates relating a newly added object (i.e. $ o \in g-\hat{\mathcal{O}}$) to its ancestors in $G$ \\
            $g \gets \textsc{Objects}(p)$

            \If{all objects $\mathcal{O}$ visited}{
                break
            }
        }
        $\hat{G} \gets G$ \tcc*[r]{Initialize sparsified scenegraph}
        \textsc{ConnectRooms}  \tcc*[r]{All-pairs shortest paths}
        Remove all nodes from $\hat{G}$ that are not in $\hat{\mathcal{O}}$ \\ %
        Prune literals that are no longer valid in the sparsified graph \\ %
        \caption{\scrub{}}
        \end{algorithm}
    \end{minipage}
\vspace{-1cm}
\end{wrapfigure}

As discussed above, learning-based planners leverage a wealth of prior knowledge acquired during a training phase to significantly prune extraneous scene graph entities.
We argue that, if equipped with the right sparsification machinery, classical planners can compete with, or outperform learning methods.
We develop \scrub{}, a principled \sg{} sparsification scheme that prunes a \sg{} $G$ (w.r.t. planning problem $\Pi_G = \langle \cal{O}, \cal{P}, \cal{A}, \cal{T}, \cal{C}, \cal{I}, \cal{G} \rangle$) by removing vertices and edges extraneous to the task, resulting in a sparsified \sg{} $\hat{G}$ (and planning problem $\hat{\Pi}_{\hat{G}} = \langle \hat{\mathcal{O}}, \hat{\mathcal{P}}, \hat{\mathcal{A}}, \hat{\mathcal{T}}, \hat{\mathcal{C}}, \hat{\mathcal{I}}, \mathcal{G} \rangle$)
\begin{definition}
A valid \sg{} sparsification of $G$ for a planning problem $\Pi_G$ to $\hat{G}$ (and corresponding planning problem $\hat{\Pi}_{\hat{G}}$) is a computable function \scrub{}$(\Pi_G) = \hat{\Pi}_{\hat{G}}$ such that, a plan $p$ solves $\Pi_g$ iff it solves $\hat{\Pi}_{\hat{G}}$.
\end{definition}
A satisficing plan for $\Pi_G$ may thus be obtained by simply solving the (much easier to solve) sparsified problem $\hat{\Pi}_{\hat{G}}$. Savings in planning time depend on the complexity of the sparsified subgraph $\hat{G}$.
\scrub{} presents a simple strategy which is guaranteed to be minimal for grounded planning problems and satisficing for lifted planning problems.

For exposition, we consider grounded planning problems; see appendix for how \scrub{} is adapted to lifted planning problems or stochastic transitions.
\scrub{}  begins with an initially empty sufficient object set $\hat{\cal{O}}$. Satisfying the goal minimally requires all ground objects in the goal to be included in the sufficient object set $\hat{\cal{O}}$ (else goal objects are unreachable). In addition, the robot itself must be part of the sufficient set. Let $p$ be the set of all binary predicates which include any of these objects. And let $g$ be the set of all objects contained in $p$. In general, this will be a superset of the objects we started with. We iteratively repeat this process, each time adding the new objects in $g$ to our sufficient set $\hat{\cal{O}}$.

The process terminates either when either the set $g$ has no new objects (indicating convergence), or until all the objects in the scene graph are visited at least once (indicating the input graph already defines a minimal object set).  %
We initialize the nodes of $\hat{G}$ with objects in $\hat{\cal{O}}$, and copy over all edges $(u, v) \in G$ for which both $u, v \in \hat{\cal{O}}$.
\scrub{} terminates in time linear in the number of the predicates or nodes (whichever is larger).

\begin{proposition}
\scrub{} is complete and results in a minimal scene subgraph for all grounded planning problems over the scenegraph domain. (Please refer to supp. material for proof)
\end{proposition}

\begin{minipage}{\textwidth}
    \begin{minipage}{.6\textwidth}
        \includegraphics[width=\linewidth]{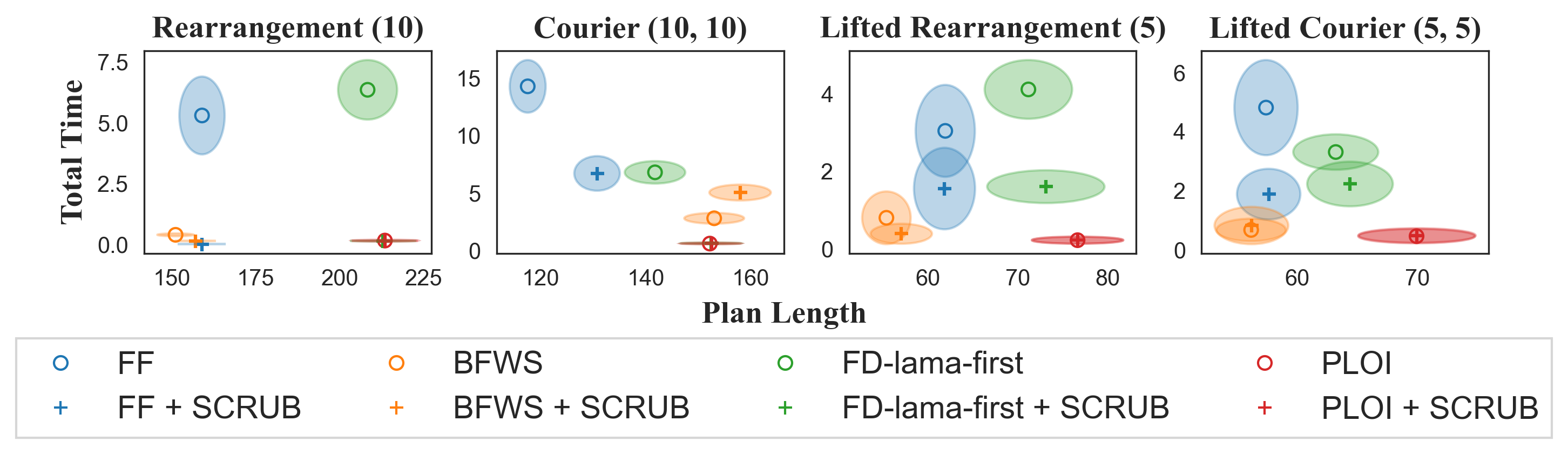}
        \captionof{figure}{Best performing planners with and without \scrub{}.}
        \label{fig:scrub-planners}
    \end{minipage}
    \begin{minipage}{.40\textwidth}
        \includegraphics[width=\linewidth]{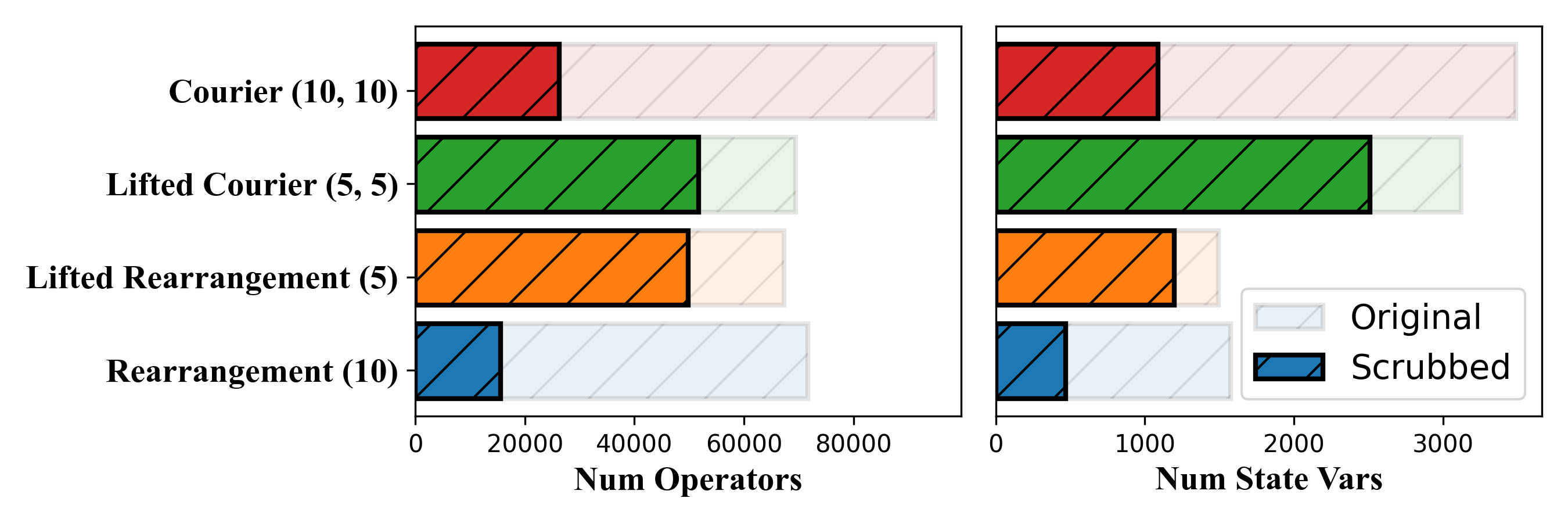}
        \captionof{figure}{\scrub{} greatly prunes operators and states of planning problems.}
        \label{fig:scrub-domain-stats}
    \end{minipage}
\end{minipage}

\subsection{Impact of \scrub{} on modern task planners}
In this section, we investigate the effect that a \sg{} reduction scheme like \scrub{} may have on the performance of modern task planners. We experiment with the four domains shown in~Fig.~\ref{fig:scrub-planners} and evaluate the impact of scrub on planning performance and on domain structure.

\textbf{\scrub{} enables classical planners to obtain performance at least as good as state-of-the-art planners}. In Fig.~\ref{fig:scrub-planners}, we see that \scrub{} drastically reduces planning time for \texttt{FF}, \texttt{FD-lama-first}, and \texttt{BFWS} to a few milliseconds on Rearrangement(10), and upper-bounds times at 5 seconds on Courier(10, 10).
We see this enables BFWS, FD, and FF to outperform PLOI (lower plan lengths for similar plan times). The grounded domains each have 182 test problems, and the lifted domains each have 70 test problems.

\textbf{\scrub{} greatly reduces the number of operators and states}. To asses the impact of \scrub{}, we compute statistics (number of operators, number of state variables) in Fig.~\ref{fig:scrub-domain-stats}.
We see that \scrub{} prunes \emph{more than two-thirds} of the operators and state variables for grounded planning problems, and about a third in the case of lifted planning problems.

\begin{wraptable}{r}{.35\textwidth}
    \vspace{-.5cm}
    \caption{Planner statistics evaluated over 70 test problems on \emph{Lifted Rearrangement}(5).}
    \adjustbox{width=.35\textwidth}{
        \begin{tabular}{cccc}
            \toprule
            \textbf{Planner} & \textbf{\% Success} & \textbf{Length} & \textbf{Time} \\
            \midrule
            \textbf{FD (satisficing)} & 51.43 & 66.81 & 3.20 \\
            \textbf{FD (satisficing) + \scrub{}} & 72.86  & 73.09  & 1.61 \\ 
            \textbf{FD (optimal)} & - & - & - \\
            \textbf{FD (optimal) + \scrub{}} & 72.86  & 68.33  & 2.26 \\
            \bottomrule
        \end{tabular}
    } %
    \label{table:scrub-optimal}
    \vspace{-.5cm}
\end{wraptable}
\textbf{\scrub{} enables optimal planners to run on lifted domains}. Table~\ref{table:scrub-optimal} reports results of running the satisficing and optimal variants of FD with and without \scrub{}, on the \emph{Lifted Rearrangement(5)} domain. While FD (optimal) did not converge even with a timeout of \emph{24 hours}, FD (optimal) + scrub solved about 72\% of the tasks under a 30-second timeout, taking 2 seconds per task on average.

\section{\seek{}: A procedure for efficient learning-based planning}
\label{sec:seek}

While \scrub{} results in a \sg{} reduction that is guaranteed to find a satisficing plan---if one exists---its conservative approach hurts performance in challenging lifted planning problems as shown in Fig.~\ref{fig:scrub-planners}.
For such problems, learning-based graph-pruning strategies like PLOI~\cite{ploi} outperform classical planners over \scrubbed{} \sgs{}.
However, as can be seen in Sec.~\ref{sec:taskography-results-discussion}, even PLOI~\cite{ploi} incurs a significant number of replanning iterations.

We posit that several replanning iterations may be avoided by exploiting the \sg{} hierarchy.
Pruning strategies like PLOI first score all objects, and retain a minimal set by thresholding.
A simple threshold does little to ensure that all retained objects are reachable from the root of the scene graph.
To alleviate this issue, we propose \seek{}: a procedure that ensures we obtain a connected graph, with the objective of reducing the number of replanning steps needed.

\seek{} requires as input the \sg{}, the planning problem $\Pi$, and an object scoring mechanism $f_\theta$. This scoring mechanism is typically a graph neural network (akin to~\cite{ploi}) that, given the current state, scores each object with an \emph{importance} value in $\left[0, 1\right]$.
We first run the scorer and only retain objects above a threshold score $t$.
We follow an identical approach to PLOI~\cite{ploi} and at each step geometrically decay the threshold by $\gamma$, such that at iteration $i$, the threshold is $t_i = \gamma t_{i-1}$, with $t_0, \gamma \in [0, 1)$.
For each retained object $o$, we recursively traverse up the \sg{}, adding all ancestors of $o$ to the sufficient object set.
This procedure ensures that all objects are reachable from their respective room nodes. While \seek{}, unlike \scrub{}, is not guaranteed to be satisficing, it results in far fewer replanning steps without affecting computation time.

\begin{wrapfigure}{r}{.4\textwidth}
    \vspace{-.7cm}
    \begin{minipage}{.4\textwidth}
        \includegraphics[width=\linewidth]{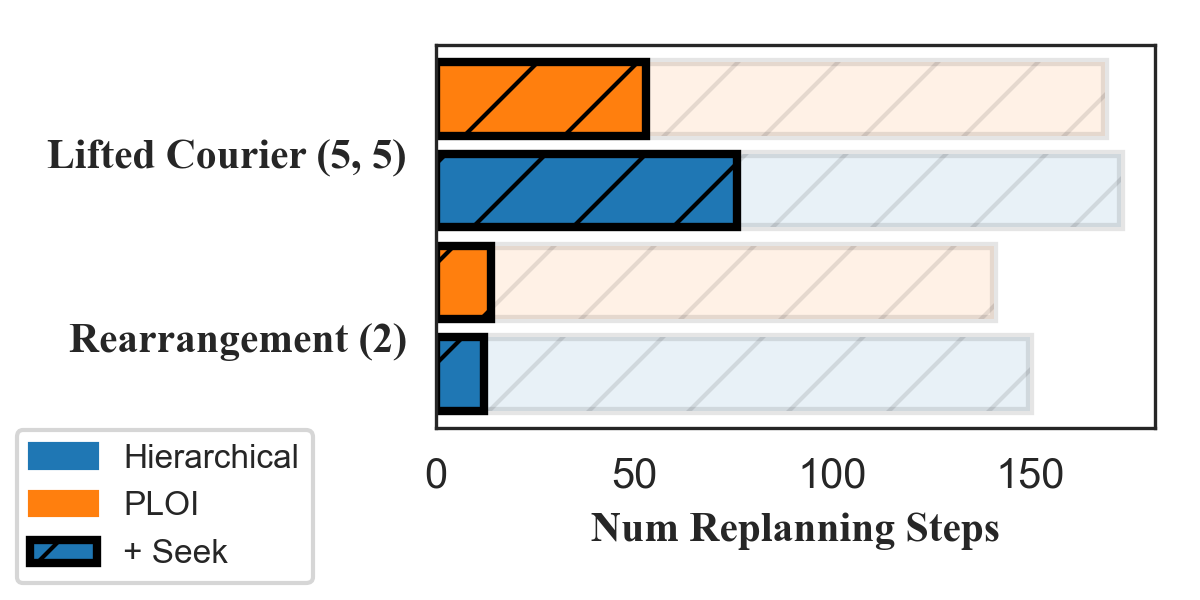}
        \captionof{figure}{\seek{} reduces replanning steps by an order of magnitude.}
        \label{fig:seek-replanning-steps}
    \end{minipage}
    \begin{minipage}{.4\textwidth}
        \includegraphics[width=\linewidth]{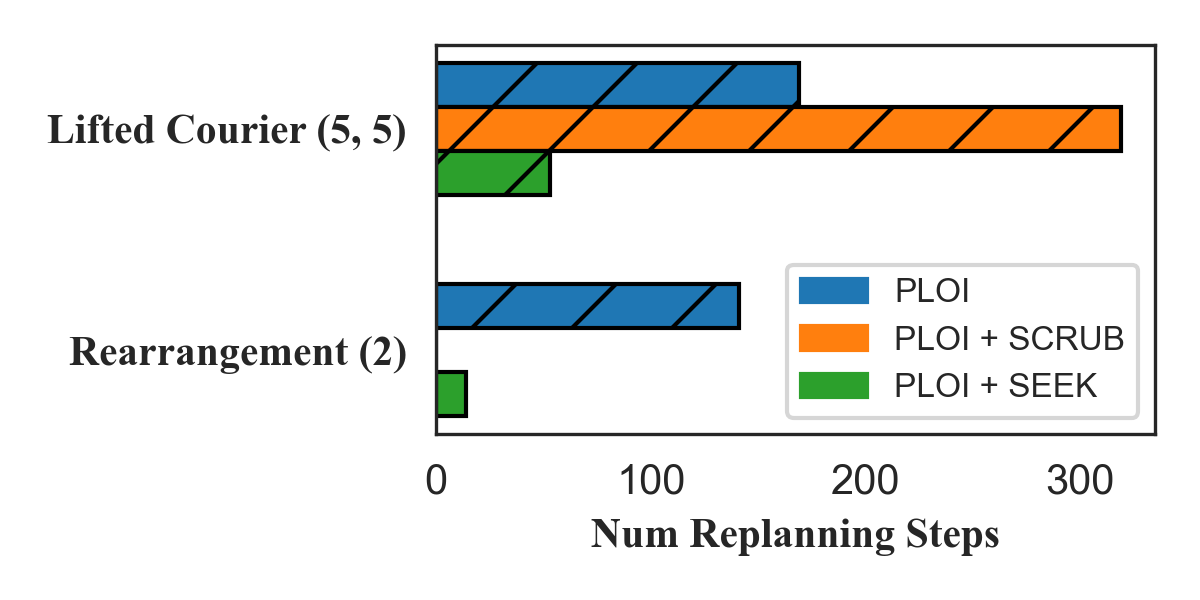}
        \captionof{figure}{\scrub{} on grounded domains, \seek{} on lifted domains.}
        \label{fig:scrub-seek-comparison}
    \end{minipage}
    \vspace{-.7cm}
\end{wrapfigure}

\textbf{\seek{} reduces replanning steps by an order of magnitude}. 
To assess the impact of the \seek{} procedure on planning performance, we evaluate performance with respect to other learning-based planners on \taskography{} in Table~\ref{table:seek-planners}.
As a baseline, we evaluate a \emph{random} pruning strategy that uniformly randomly retains or prunes every object. Even for this naive strategy, \seek{} offers significant performance improvement.
We also evaluate \emph{PLOI}~\cite{ploi} and our adaptation dubbed \emph{hierarchical}, which trains multiple graph neural networks, one for each level of the \sg{} hierarchy.
For each variant, \seek{} offers a consistent performance improvement by decreasing the number of replanning steps required as seen in Fig.~\ref{fig:seek-replanning-steps}.
\seek{} is thus a conceptually simple strategy for use with learning-based planners.

\textbf{\scrub{} on grounded domains, \seek{} on lifted domains}: In general, we note that \scrub{} is more performant on grounded domains (due to minimality properties) and \seek{} is more performant on lifted domains (where \scrub{} typically retains all instances of important object categories, but \seek{} is more effective due to its opportunistic retention of instances (Fig.~\ref{fig:scrub-seek-comparison})).

\newcolumntype{R}{@{\extracolsep{1.3cm}}r@{\extracolsep{10pt}}}%

\begin{table}[]
\caption{\textbf{\seek{}} significantly reduces the number of replanning steps required by state-of-the-art learning-based planners. For each planner, we report average \emph{wall time} (including translation time).}
\adjustbox{max width=\linewidth}{
\begin{tabular}{@{}lrrrrrRrrrrRrrrrRrrrr@{}}
\toprule
\multicolumn{1}{l}{\multirow{2}{*}{\textbf{Planner}}} & \multicolumn{5}{c}{\textbf{Rearrangement (2) - Medium}}                                          & \multicolumn{5}{c}{\textbf{Courier (10, 3) - Medium}}                                            & \multicolumn{5}{c}{\textbf{Lifted Rearrangement (5) - Medium}}                                   & \multicolumn{5}{c}{\textbf{Lifted Courier (5, 5) - Medium}}\\
\cmidrule{3-5} \cmidrule{8-10} \cmidrule{13-15} \cmidrule{18-20}
\multicolumn{1}{c}{}                                  & \textbf{\%Succ.} & \textbf{Len.} & \textbf{\%Used} & \textbf{Time} & \textbf{\#Replan} & \textbf{\%Succ.} & \textbf{Len.} & \textbf{\%Used} & \textbf{Time} & \textbf{\#Replan} & \textbf{\%Succ.} & \textbf{Len.} & \textbf{\%Used} & \textbf{Time} & \textbf{\#Replan} & \textbf{\%Succ.} & \textbf{Len.} & \textbf{\%Used} & \textbf{Time} & \textbf{\#Replan} \\
\cmidrule{2-6} \cmidrule{7-11} \cmidrule{12-16} \cmidrule{17-21}
\textbf{Random}                                         & 0.87              & 39.81        & 0.99           & 9.51         & 836               & 0.62              & 180           & 0.10           & 12.11        & 204               & 0.63              & 68.98        & 0.99           & 10.93        & 235               & 0.67              & 67.89        & 0.98           & 10.81        & 233               \\
\textbf{Random + \seek{}}                                  & 0.86               & 39.82        & 0.98           & 8.55         & \textbf{543}               & 0.60                & 183.49       & 0.99           & 12.33        & \textbf{162}               & 0.59              & 69.22         & 0.97           & 9.52         & \textbf{155}               & 0.63              & 65.48        & 0.97           & 10.97        & \textbf{167}               \\
\midrule
\textbf{Hierarchical}                                   & 1                  & 35.76        & 0.28           & 0.45         & 150               & 1                  & 191.75       & 0.48           & 1.16         & 40                & 0.80                & 76.75         & 0.59           & 2.60         & 269               & 0.73              & 69.69        & 0.61           & 2.73         & 173               \\
\textbf{Hierarchical + \seek{}}                            & 1                  & 35.76        & 0.28           & 0.30           & \textbf{12}                & 1                  & 191.75       & 0.48           & 0.97         & \textbf{7}                 & 0.80                & 76.70        & 0.56           & 2.20         & \textbf{208}               & 0.77              & 76.04        & 0.55           & 1.59         & \textbf{76}                \\
\midrule
\textbf{PLOI~\cite{ploi}}                                           & 1                  & 35.76        & 0.28           & 0.44         & 141               & 1                  & 191.75       & 0.48           & 1.13         & 41                & 0.79              & 78.16        & 0.59           & 2.49         & 258               & 0.73              & 69.88        & 0.62           & 2.75         & 169               \\
\textbf{PLOI + \seek{}}                                    & 1                  & 35.76        & 0.28           & 0.31         & \textbf{14}                & 1                  & 191.75       & 0.48           & 0.97         & \textbf{7}                 & 0.80                & 76.61        & 0.56           & 2.18         & \textbf{197}               & 0.77              & 79.19        & 0.55           & 1.53         & \textbf{53}                \\ \bottomrule
\end{tabular}
} %
\vspace{-.5cm}
\label{table:seek-planners}
\end{table}

\vspace{-.35cm}
\section{Concluding remarks}
\label{sec:discussion}
\vspace{-.35cm}

\textbf{Limitations.} \taskography{} currently supports only a fraction of the diverse types of planning problems possible on \sgs{}. Geared towards identifying the most promising avenues in learning-based planning, the first release of this benchmark focuses exclusively on offline task planning in fully observable and deterministic domains.
Furthermore, low-level motion planning is excluded from our benchmark. Robots operating in the real world will need to reason under partial observability, sensor noise, and resource constraints.

\textbf{Outlook.} \taskography{}, in conjunction with \scrub{} and \seek{} aid the robot learning community by (a) providing guidelines and implementations for practitioners choosing a task planner, (b) serving as a benchmark for upcoming learning-based planners, and (c) guiding the design of futuristic spatial representations for robotic task planning.
We believe \taskography{} is a first step towards addressing several of the grand challenges along the road to developing general planning capabilities for autonomous intelligent robots.

\section*{Acknowledgements}

CA and KMJ would like to thank \authorhref{https://web.mit.edu/tslvr/www/}{Tom Silver} and \authorhref{https://rohanchitnis.com/}{Rohan Chitnis} for their help at various stages of this project, including early-stage feedback, code release, and for proofreading an initial version of this manuscript. KMJ acknowedges generous fellowship support from NVIDIA. LP acknowledges grants from IVADO and from the CIFAR Canada AI chairs program. FS acknowledges funding support from NSERC and the NFRF exploration program. The authors collectively acknowledge support during various initial stages of the project by \authorhref{https://fgolemo.github.io/}{Florian Golemo}.

\bibliography{ms}  %

\appendix

In the appendix, we discuss additional details and design choices for the \taskography{} benchmark, including extended descriptions for all supported planning domains and their constituents - object types, relations (i.e., properties, predicates), and goal specifications. 
We provide results these additional domains, and discuss \scrub{} and its favourable properties in greater detail.
    
\textbf{Please visit our \href{https://taskography.github.io}{project page} for more details, including a \href{https://www.youtube.com/watch?v=mM4v5hP4LdA}{video abstract}.}

\section{Benchmark Details}

The \taskography{} benchmark comprises 20 robot task planning domains over 3D scene graphs (\sgs{}). In the main paper, we detailed the \emph{Rearrangement(k)}, \emph{Courier(n,k)}, \emph{Lifted Rearrangement(k)}, and \emph{Lifted Courier(n, k)} task definitions following the recently proposed Rearrangement challenge~\cite{rearrangement}.
Table.~\ref{table:taskography-domain-objects} lists the set of \emph{lifted} objects in each planning domain. In all problems, we have one instance of an \emph{agent}, but several ground objects corresponding to the other categories.

\subsection{\taskography{} domain construction: Parsing Gibson 3DSGs}
\label{subsec:parsing-scene-graphs}

We parse the \sgs{} created over Gibson~\cite{3dscenegraph, gibson} mapping scene entities to objects and structural relations to predicates over objects. 
We retain key connectivity constraints that govern traversable paths between locations in the same place, places in the same room, and between rooms. 
Because room connectivity data not is provided in the original database, we estimate it by computing a minimal spanning tree over rooms in the \sgs{} with edge weights reflecting the Euclidean distance between room centroids.
For larger scenes, we impose a single connection between rooms in different floors (e.g., one set of stairs).
Several additional properties are used to express the state of agent and interactable objects, and to associate each of them to a particular location in the \sg{}. 

\vspace{.2cm}
\begin{table}[!h]
    \centering
    \caption{Evaluated \sg{} planning domains in \textsc{Taskography} and object types present in each. Domains are further partitioned into tiny and medium splits akin to the \sgs{} provided over Gibson \cite{3dscenegraph, gibson}. Scene entities are instantiated as a particular object type according to their semantic class.}
    \adjustbox{max width=\linewidth}{
    \begin{tabular}{@{}l|cc|ccccccccc@{}}
    \toprule
    & \textbf{n} & \textbf{k} & \textbf{Agent} & \textbf{Room} & \textbf{Place} & \textbf{Location} & \textbf{Receptacle} & \textbf{Item} & \textbf{Bagslot}  & \textbf{Receptacle Class} & \textbf{Item Class} \\
    \midrule
    \textbf{Rearr(k)} & - & \{1, 2, 5, 10\} & \cmark & \cmark  & \cmark  & \cmark  & \cmark  & \cmark  & \xmark  & \xmark  & \xmark \\
    \textbf{Cour(n, k)} & \{3, 5, 7, 10\} & \{5, 10\} & \cmark & \cmark  & \cmark  & \cmark  & \cmark  & \cmark  & \cmark & \xmark  & \xmark \\
    \textbf{Lifted Rearr(k)} & - & \{5\} & \cmark & \cmark  & \cmark  & \cmark  & \cmark  & \cmark  & \xmark  & \cmark  & \cmark \\
    \textbf{Lifted Cour(n, k)} & \{5\} & \{5\} & \cmark & \cmark  & \cmark  & \cmark  & \cmark  & \cmark  & \cmark  & \cmark  & \cmark \\
    \bottomrule
    \end{tabular}
    } %
    \label{table:taskography-domain-objects}
\end{table}

\vspace{.2cm}

An assignment of values to all possible properties over objects defines a symbolic \emph{state} in the planning problem; hence, actions taken by the robot in \taskography{} alter the symbolic state of the \sg{}.  
We observe a significant variation in the size of the state space between different types of domains as a result of the varying subsets of object and predicate types used to express their respective tasks (see Table.~\ref{table:taskography-domain-predicates}). 
For instance, the \emph{Rearrangement(k)} task represents the lowest complexity domain on \taskography{} and is thereby defined by the smallest subset of object types, predicates, and actions available to the robot.
In contrast, the \emph{Lifted Courier(n, k)} extends the \emph{Rearrangement(k)} task definition with bagslots enabling stow and retrieve operators, as well as receptacle classes and item classes to express lifted class relations in the \sg{} at particular state.

We leverage \textbf{task samplers} built into \taskographyapi{} for generating large-scale and diverse datasets of planning problems over \sgs{}. In a two step process the task samplers automatically parse \sgs{} into plannable symbolic representations (i.e., embedding the agent forms the initial state \(\mathcal{I}\)) before composing goal literals over randomly sampled scene entities.
For grounded problems, goals are conjunctions of \graytext{inReceptacle} literals expressed over randomly sampled item and receptacle target ground instances.
For lifted problems, goal are conjunctions of \graytext{classRelation} literals expressed over randomly sampled item and receptacle target class relations.

\begin{table}[!h]
    \centering
    \caption{Structural relations of \sgs{} and the state of the robot and interactable objects (i.e., items and receptacles) are captured with an assignment of values to all possible predicates over objects. The most challenging \emph{Lifted Courier(n, k)} is the only domain to incorporate all relations, while other domain types in \textsc{Taskography} require only a subset of the properties and relations.}
    \adjustbox{max width=\linewidth}{
    \begin{tabular}{@{}l|ccccccccc@{}}
    \toprule
    \textbf{Object} (:types) & \textbf{Agent} & \textbf{Room} & \textbf{Place} & \textbf{Location} & \textbf{Receptacle} & \textbf{Item} & \textbf{Bagslot}  & \textbf{Receptacle Class} & \textbf{Item Class} \\
    \midrule
    \textbf{Agent} & holdsAny & inRoom & inPlace & atLoc & -  & holdsItem & -  & -  & - \\
    \textbf{Room} & inRoom & connected & placeInRoom + roomCenter  & - & -  & -  & -  & -  & - \\
    \textbf{Place} & inPlace & placeInRoom + roomCenter & -  & locInPlace + placeCenter  & -  & -  & -  & -  & - \\
    \textbf{Location} & atLoc & -  & locInPlace + placeCenter  & -  & recepAtLoc  & itemAtLoc  & -  & -  & - \\
    \textbf{Receptacle} & - & - & - & recepAtLoc  & recepOpened  & inRecep  & -  & recepClass  & - \\
    \textbf{Item} & holdsItem & -  & -  & itemAtLoc  & inRecep  & small + medium + large  & inSlot  & -  & itemClass \\
    \textbf{Bagslot} & - & -  & -  & -  & -  & inSlot  & slotHoldsAny  & -  & - \\
    \textbf{Receptacle Class} & - & -  & -  & -  & recepClass  & -  & -  & -  & classRelation \\
    \textbf{Item Class} & - & -  & -  & -  & -  & itemClass  & -  & classRelation  & - \\
    
    \bottomrule
    \end{tabular}
    } %
    \label{table:taskography-domain-predicates}
\end{table}

\subsection{Domain specifications}
\label{subsec:elaborating-task-specifications}

To provide further clarity on the four task categories (\emph{Rearrangement(k)}, \emph{Courier(n,k)}, \emph{Lifted Rearrangement(k)}, and \emph{Lifted Courier(n, k)}) from which our \sg{} planning domains are constructed, we herein outline hypothetical problem instances involving but a fraction of the objects, attributes, and relations available in \taskography{}.
Let the environment consist of \emph{v} \texttt{rooms} connected by \emph{e} undirected traversability constraints; e.g., \graytext{connected(roomA, roomB)}.
The spatial hierarchy of \sgs{} \cite{3dscenegraph, 3ddynamicscenegraphs} is induced by the appropriate application of structural relations (see Table.~\ref{table:taskography-domain-predicates}) to a discrete set of \texttt{places} in each \texttt{room}, and \texttt{locations} in each \texttt{place}; e.g., \graytext{placeInRoom(placeD, roomC)}, \graytext{locInPlace(locF, placeD)}.
The lowest level of the spatial hierarchy (\texttt{locations}) encodes all occupiable positions for the \texttt{agent}, \texttt{items}, and \texttt{receptacles} in the scene; e.g., \graytext{atLoc(agent, locationB)}, \graytext{itemAtLoc(mugA, locationD)}, \graytext{recepAtLoc(fridgeC, locationG)}.
Such relations equate to logical predicates in \cite{pddl} and can be altered by the agent should the required preconditions of an action be met in the current state; e.g., \graytext{$\neg$holdsAny(agent)} and $\wedge(\text{\graytext{atLoc(agent, locX)}, \graytext{itemAtLoc(mugA, locX)}})$ are preconditions for \textsc{PickupItem}(\graytext{mugA}, \graytext{agent}).

As mentioned in Sec.~\ref{subsec:parsing-scene-graphs}, the goals in grounded planning problems are specified with \graytext{inReceptacle} literals. 
Concretely, a \emph{Rearrangement(k)} task for $k=1$ requires the agent to pick-and-place a ground item in a ground receptacle, where each object in the goal is uniquely identified; e.g., $G=\text{\graytext{inReceptacle(mugA, fridgeC)}}$.
By extension, a \emph{Rearrangement(k)} task for $k=2$ is solved \emph{iff} the agent derives a state satisfying the conjunction of two \graytext{inReceptacle} goal literals; e.g., $G=\wedge(\text{\graytext{inReceptacle(mugA, fridgeC)}, \graytext{inReceptacle(plateD, shelfB)}})$. 
The \emph{Courier(n, k)} domains attribute weights ($w\in{1, 2, 3}$ units) to items based on their volume, and equips the agent with a knapsack of fixed capacity \emph{n} to stow and retrieve items as it traverses the scene.
While the knapsack in \emph{Courier(n, k)} enables planners to exploit stowing capacity to compute lower cost solutions (at the expense of task complexity) in comparison to \emph{Rearrangement(k)}, goals are identically specified between the two task categories since they are both considered grounded.

In stark constrast, lifted planning problems are specified with \graytext{classRelation} literals expressed over item-receptacle class combinations. 
For instance, the following \emph{Lifted Rearrangement(k)} or \emph{Lifted Courier(n, k)} domain with $k=2$, $G=\wedge(\text{\graytext{classRelation(cup, cupboard)}, \graytext{classRelation(plate, sink)}})$, requires the agent to place \textbf{at least one} cup in a cupboard and plate in a sink for the task to be complete.
This disambiguates the planner which is no longer able to exploit ground objects featured in the goal as heuristic \emph{landmarks}, and reduces the effectiveness of deterministic graph sparsification techniques such as \textsc{SCRUB}. 
As in the grounded domain variants, the goal specifications for both the \emph{Lifted Rearrangement(k)} and \emph{Lifted Courier(n, k)} are identical.

\subsection{Symbolic environment interaction}
\label{subsec:symbolic-env-interaction}

The \textbf{action space} of the most complex domain in \taskography{} equips the agent with 16 operators where only a subset are feasible at any given state. 
Below, we describe but a few of these operators which demonstrate motion through \sg{} hierachies and object-level robot interaction.
\begin{compactitem}
    \item \textsc{GoToRoom}: The robot moves from the door of its current room to the door of the target room if the rooms are \emph{connected}.
    \item \textsc{GoToPlace}: The robot moves from the center of its current place to the center of the target place if the places are in the same room.
    \item \textsc{GoToLocation}: The robot moves from the current location to the target location if the locations are in the same place.
    \item \textsc{OpenReceptacle}: The robot opens a queried \emph{openable} receptacle.
    \item \textsc{CloseReceptacle}: The robot closes a queried \emph{openable} receptacle.
    \item \textsc{PickupItem}: The robot picks-up an item at a particular location with a free gripper; three operator variations for picking from non existent, non-opening, and opening receptacles.
    \item \textsc{PlaceItem}: The robot places an in-gripper item at a particular location; two operator variations for placing in non-opening and opening receptacles.
    \item \textsc{StowItem}: The robot stows an in-gripper item in its knapsack: three operator variations for small, medium, and large items consuming increasing numbers of bagslots. 
    \item \textsc{RetrieveItem}: The robot retrieves an item from its knapsack into its gripper; three operator variations for small, medium, and large items freeing increasing number of bagslots.
\end{compactitem}
Should the preconditions for any of these actions not be satisfied, the action is deemed invalid.

\section{\scrub{}: Discussion and analysis}

In the main paper, for sake of brevity, we only discussed the applicability of \scrub{} to grounded planning problems with deterministic transitions. However, by design, \scrub{} may be applied to any planning problem: \emph{lifted} or \emph{grounded}, with \emph{deterministic} or \emph{stochastic} transitions.

In \emph{lifted} planning problems, we modify \scrub{} to trivially include all ground object tuples that satisfy goal conditions into the initial sufficient object set. This in-turn ensures that all of these ground objects are reachable from the start state, ensuring a satisficing plan exists. However, this conservative strategy may resulting in retaining more objects than minimally required -- this is where \seek{} can be applied to opportunistically retain important objects instead.

In a similar vein, for \emph{stochastic} transitions, we modify \scrub{} to include all binary predicates resulting from all possible stochastic transitions from a given node.

We now prove that \scrub{} results in a minimal scene subgraph for all grounded planning problems.

\begin{proposition}
\scrub{} is complete and results in a minimal scene subgraph for all grounded planning problems over the scenegraph domain.
\end{proposition}
\begin{proof}
We prove the minimality of \scrub{} by demonstrating that whenever we prune a node from a \scrubbed{} scenegraph, the resultant planning problem is unsolvable.
Assume that we prune a node $n$ from a \scrubbed{} \sg{} $\hat{G}$.
Recall the types of nodes we have in the \sg{}: \texttt{agent, room, place, receptacle, item, floor, building} .
\begin{enumerate}
    \item If $n$ is of type \texttt{agent} or \texttt{building}, the problem is unsolvable, by construction.
    \item If $n$ is of type \texttt{item}, removing it would render the goal state unreachable --- recall that $\hat{G}$ only retains \texttt{item} nodes that feature in the goal state.
    \item If $n$ is of type \texttt{receptacle}, it is retained in $\hat{G}$ either because (a) it is required to access a goal object of type \texttt{item}, or (b) it is a goal \texttt{receptacle} (i.e., a target location an \texttt{item} must be moved into). Removing $n$ will thus render one of the objects in the goal state unreachable.
    \item If $n$ is of type \texttt{place}, \texttt{room} or \texttt{floor}, $n \in \hat{G}$ because $n$ directly features in the goal state, or because $n$ is required to traverse from the start state to the goal state (e.g., rooms that connect the start and goal rooms, etc.).
\end{enumerate}
Since pruning any of these nodes renders the problem unsolvable, the \scrubbed{} graph $\hat{G}$ is a minimal scene subgraph for the grounded planning problem considered.
\end{proof}
\section{Additional results on \taskography{} domains}

In this section, we provide results over several extended domains from the \taskography{} benchmark. Please see Tables~\ref{table:taskography-rearrangement-tiny}, \ref{table:taskography-rearrangement-medium}, \ref{table:taskography-courier-tiny}, \ref{table:taskography-courier-medium}, \ref{table:taskography-lifted-rearrangement}, \ref{table:taskography-lifted-courier}.

\begin{table}[!h]
    \centering
    \caption{Performance of planners over the \emph{Rearrangement(k)}-Tiny tasks. For all metrics, lower values indicate better performance.}
    \adjustbox{max width=\linewidth}{
    \begin{tabular}{@{}llrrrrrrrrr@{}}
    \toprule
    &  & \multicolumn{3}{c}{\textbf{Rearr(1)} Tiny} & \multicolumn{3}{c}{\textbf{Rearr(2)} Tiny} & \multicolumn{3}{c}{\textbf{Rearr(10)} Tiny} \\
    \cmidrule{3-5} \cmidrule{6-8} \cmidrule{9-11}
     & \textbf{Planner}          &    \textbf{Len.} &    \textbf{Time} &   \textbf{Fail} &    \textbf{Len.} &    \textbf{Time} &   \textbf{Fail} &    \textbf{Len.} &    \textbf{Time} &   \textbf{Fail} \\
    \midrule
    \multirow{5}{*}{\STAB{\rotatebox[origin=c]{90}{optimal}}} & \textbf{FD-seq-opt-lmcut} & 15.77  & 24.81  &   0.04 & \textbf{25.80}  & 104.47 &   0.55 & -      & -      &   1.00 \\
    & \textbf{SatPlan}          & 14.77  & 10.35  &   0.45 & 26.67  & 3.27   &   0.67 & -      & -      &   1.00 \\
    & \textbf{Delfi}            & 15.13  & 0.36   &   0.16 & 29.10  & 27.77  &   0.29 & -      & -      &   1.00 \\
    & \textbf{DecStar-opt-fb}   & -      & -      &   1.00 & -      & -      &   1.00 & -      & -      &   1.00 \\
    & \textbf{MCTS}             & -      & -      &   1.00 & -      & -      &   1.00 & -      & -      &   1.00 \\
    \midrule
    \multirow{6}{*}{\STAB{\rotatebox[origin=c]{90}{satisficing}}} & \textbf{FF}               & 16.71  & 0.19   &   \textbf{0.00} & 34.44  & 0.55   &   \textbf{0.00} & 162.61 & 7.04   &   0.07 \\
    & \textbf{FF-X}             & 16.71  & 0.25   &   \textbf{0.00} & 34.44  & 0.58   &   \textbf{0.00} & 162.30 & 7.39   &   0.09 \\
    & \textbf{FD-lama-first}    & 15.19  & 2.96   &   0.33 & 38.47  & 3.25   &   0.18 & 205.89 & 7.68   &   0.51 \\
    & \textbf{Cerberus-sat}     & \textbf{11.50}  & 12.00  &   0.85 & -      & -      &   1.00 & -      & -      &   1.00 \\
    & \textbf{Cerberus-agl}     & 14.77  & 5.13   &   0.45 & 33.00  & 7.30   &   0.49 & 186.07 & 9.04   &   0.73 \\
    & \textbf{DecStar-agl-fb}   & 14.72  & 2.62   &   0.55 & 34.96  & 2.58   &   0.58 & 193.00 & 6.78   &   0.85 \\
    & \textbf{BFWS}             & 15.56  & 0.90   &   0.22 & 32.16  & 0.37   &   0.18 & \textbf{160.93} & 0.57   &   0.18 \\
    & \textbf{Regression-plan}  & -      & -      &   1.00 & -      & -      &   1.00 & -      & -      &   1.00 \\
    \midrule
    \multirow{2}{*}{\STAB{\rotatebox[origin=c]{90}{learn}}} & \textbf{Relational policy~\cite{relationalpolicy}} & -      & -      &   1.00 & -      & -      &   1.00 & -      & -      &   1.00 \\
    & \textbf{PLOI~\cite{ploi}}             & 16.45  & \textbf{0.00*}   &   \textbf{0.00} & 37.04  & \textbf{0.00*}   &   \textbf{0.00} & 221.71 & \textbf{0.18}   &   \textbf{0.00} \\
    \bottomrule
    \end{tabular}
    } %
    \label{table:taskography-rearrangement-tiny}
\end{table}

\begin{table}[!h]
    \centering
    \caption{Performance of planners over the \emph{Rearrangement(k)}-Medium tasks. For all metrics, lower values indicate better performance.}
    \adjustbox{max width=\linewidth}{
    \begin{tabular}{@{}llrrrrrrrrr@{}}
    \toprule
    &  & \multicolumn{3}{c}{\textbf{Rearr(1)} Medium} & \multicolumn{3}{c}{\textbf{Rearr(2)} Medium} & \multicolumn{3}{c}{\textbf{Rearr(10)} Medium} \\
    \cmidrule{3-5} \cmidrule{6-8} \cmidrule{9-11}
     & \textbf{Planner}          &    \textbf{Len.} &    \textbf{Time} &   \textbf{Fail} &    \textbf{Len.} &    \textbf{Time} &   \textbf{Fail} &    \textbf{Len.} &    \textbf{Time} &   \textbf{Fail} \\
    \midrule
    \multirow{5}{*}{\STAB{\rotatebox[origin=c]{90}{optimal}}} & \textbf{FD-seq-opt-lmcut} &  15.53 &  19.68 &   0.06 & \textbf{27.13}  & 125.69 &   0.41 & -      & -      &   1.00 \\
    & \textbf{SatPlan}          &  14.98 &  11.98 &   0.33 & 28.23  & 5.45   &   0.50 & -      & -      &   1.00 \\
    & \textbf{Delfi}            &  15.40 &   3.62 &   0.16 & 29.13  & 12.79  &   0.28 & -      & -      &   1.00 \\
    & \textbf{DecStar-opt-fb}   &  15.42 &  41.35 &   0.93 & 28.50  & 111.53 &   0.91 & -      & -      &   1.00 \\
    & \textbf{MCTS}             & -      & -      &   1.00 & -      & -      &   1.00 & -      & -      &   1.00 \\
    \midrule
    \multirow{6}{*}{\STAB{\rotatebox[origin=c]{90}{satisficing}}} & \textbf{FF}               &  16.45 &   0.25 &   \textbf{0.00} & 32.87  & 0.41   &   \textbf{0.00} & 159.04 & 5.30   &   0.09 \\
    & \textbf{FF-X}             &  16.45 &   0.21 &   \textbf{0.00} & 32.87  & 0.45   &   \textbf{0.00} & 159.80 & 5.02   &   0.08 \\
    & \textbf{FD-lama-first}    &  15.51 &   2.48 &   0.21 & 39.20  & 2.77   &   0.20 & 208.28 & 6.35   &   0.49 \\
    & \textbf{Cerberus-sat}     &  \textbf{11.20} &  10.17 &   0.88 & -      & -      &   1.00 & -      & -      &   1.00 \\
    & \textbf{Cerberus-agl}     &  15.18 &   6.10 &   0.34 & 32.20  & 6.40   &   0.33 & 176.60 & 8.91   &   0.72 \\
    & \textbf{DecStar-agl-fb}   &  15.36 &   2.15 &   0.58 & 36.35  & 2.40   &   0.59 & 211.16 & 7.20   &   0.82 \\
    & \textbf{BFWS}             &  15.42 &   0.60 &   0.23 & 30.65  & 0.44   &   0.27 & \textbf{151.17} & 0.41   &   0.23 \\
    & \textbf{Regression-plan}  & -      & -      &   1.00 & -      & -      &   1.00 & -      & -      &   1.00 \\
    \midrule
    \multirow{2}{*}{\STAB{\rotatebox[origin=c]{90}{learn}}} & \textbf{Relational policy~\cite{relationalpolicy}} & -      & -      &   1.00 & -      & -      &   1.00 & -      & -      &   1.00 \\
    & \textbf{PLOI~\cite{ploi}}             &  16.44 &   \textbf{0.00*} &   \textbf{0.00} & 36.19  & \textbf{0.00*}   &   \textbf{0.00} & 213.43 & \textbf{0.17}   &   \textbf{0.00} \\
    \bottomrule
    \end{tabular}
    } %
    \label{table:taskography-rearrangement-medium}
\end{table}

\begin{table}[!h]
    \centering
    \caption{Performance of planners over the \emph{Courier(n, k)}-Tiny tasks. For all metrics, lower values indicate better performance.}
    \adjustbox{max width=\linewidth}{
    \begin{tabular}{@{}llrrrrrrrrrrrr@{}}
    \toprule
    &  & \multicolumn{3}{c}{\textbf{Cour(3, 10)} Tiny} & \multicolumn{3}{c}{\textbf{Cour(5, 10)} Tiny} & \multicolumn{3}{c}{\textbf{Cour(7, 10)} Tiny} & \multicolumn{3}{c}{\textbf{Cour(10, 10)} Tiny} \\
    \cmidrule{3-5} \cmidrule{6-8} \cmidrule{9-11} \cmidrule{12-14}
     & \textbf{Planner}          &    \textbf{Len.} &    \textbf{Time} &   \textbf{Fail} &    \textbf{Len.} &    \textbf{Time} &   \textbf{Fail} &    \textbf{Len.} &    \textbf{Time} &   \textbf{Fail} &    \textbf{Len.} &    \textbf{Time} &   \textbf{Fail} \\
    \midrule
    \multirow{6}{*}{\STAB{\rotatebox[origin=c]{90}{satisficing}}} & \textbf{FF}               & 146.35 & 7.57   &   0.13 & 136.38 & 7.97   &   0.33 & 127.88 & 6.84   &   0.55 & 124.93 & 14.62  &   0.73 \\
    & \textbf{FF-X}             & 144.80 & 8.34   &   0.11 & 137.05 & 7.49   &   0.31 & 128.42 & 8.34   &   0.53 & 126.31 & 15.21  &   0.71 \\
    & \textbf{FD-lama-first}    & 175.15 & 8.31   &   0.53 & 159.64 & 7.31   &   0.55 & 156.12 & 6.97   &   0.55 & 145.00 & 7.50   &   0.56 \\
    & \textbf{Cerberus-sat}     & -      & -      &   1.00 & -      & -      &   1.00 & -      & -      &   1.00 & -      & -      &   1.00 \\
    & \textbf{Cerberus-agl}     & \textbf{137.87} & 10.79  &   0.73 & 127.30 & 17.61  &   0.82 & 138.25 & 21.65  &   0.93 & -      & -      &   1.00 \\
    & \textbf{DecStar-agl-fb}   & 140.47 & 4.52   &   0.69 & \textbf{124.62} & 4.65   &   0.71 & \textbf{120.20} & 4.04   &   0.73 & \textbf{117.73} & 6.98   &   0.73 \\
    & \textbf{BFWS}             & 160.18 & 1.19   &   0.18 & 159.17 & 0.94   &   0.25 & 159.90 & 1.80   &   0.29 & 153.93 & 4.28   &   0.45 \\
    & \textbf{Regression-plan}  & -      & -      &   1.00 & -      & -      &   1.00 & -      & -      &   1.00 & -      & -      &   1.00 \\
    \midrule
    \multirow{2}{*}{\STAB{\rotatebox[origin=c]{90}{learn}}} & \textbf{Relational policy~\cite{relationalpolicy}} & -      & -      &   1.00 & -      & -      &   1.00 & -      & -      &   1.00 & -      & -      &   1.00 \\
    & \textbf{PLOI~\cite{ploi}}             & 193.55 & \textbf{0.22}   &   \textbf{0.00} & 179.36 & \textbf{0.26}   &   \textbf{0.00} & 172.87 & \textbf{0.37}   &   \textbf{0.00} & 167.38 & \textbf{0.71}   &   \textbf{0.00} \\
    \bottomrule
    \end{tabular}
    } %
    \label{table:taskography-courier-tiny}
\end{table}

\begin{table}[!h]
    \centering
    \caption{Performance of planners over the \emph{Courier(n, k)}-Medium tasks. For all metrics, lower values indicate better performance.}
    \adjustbox{max width=\linewidth}{
    \begin{tabular}{@{}llrrrrrrrrrrrr@{}}
    \toprule
    &  & \multicolumn{3}{c}{\textbf{Cour(3, 10)} Medium} & \multicolumn{3}{c}{\textbf{Cour(5, 10)} Medium} & \multicolumn{3}{c}{\textbf{Cour(7, 10)} Medium} & \multicolumn{3}{c}{\textbf{Cour(10, 10)} Medium} \\
    \cmidrule{3-5} \cmidrule{6-8} \cmidrule{9-11} \cmidrule{12-14}
     & \textbf{Planner}          &    \textbf{Len.} &    \textbf{Time} &   \textbf{Fail} &    \textbf{Len.} &    \textbf{Time} &   \textbf{Fail} &    \textbf{Len.} &    \textbf{Time} &   \textbf{Fail} &    \textbf{Len.} &    \textbf{Time} &   \textbf{Fail} \\
    \midrule
    \multirow{6}{*}{\STAB{\rotatebox[origin=c]{90}{satisficing}}} & \textbf{FF}               & \textbf{141.89} & 4.94   &   0.07 & 133.46 & 6.29   &   0.20 & 128.41 & 6.62   &   0.24 & 117.50 & 14.27  &   0.78 \\
    & \textbf{FF-X}             & \textbf{141.89} & 4.47   &   0.07 & 133.50 & 5.80   &   0.19 & 128.19 & 6.72   &   0.24 & 118.67 & 15.52  &   0.77 \\
    & \textbf{FD-lama-first}    & 180.38 & 7.11   &   0.40 & 166.35 & 6.27   &   0.45 & 156.34 & 4.92   &   0.29 & 141.75 & 6.80   &   0.63 \\
    & \textbf{Cerberus-sat}     & -      & -      &   1.00 & -      & -      &   1.00 & -      & -      &   1.00 & -      & -      &   1.00 \\
    & \textbf{Cerberus-agl}     & 148.41 & 10.17  &   0.74 & \textbf{133.31} & 11.50  &   0.77 & \textbf{125.73} & 12.99  &   0.83 & \textbf{109.56} & 15.58  &   0.95 \\
    & \textbf{DecStar-agl-fb}   & 154.07 & 6.45   &   0.66 & 142.42 & 4.01   &   0.61 & 132.60 & 4.50   &   0.58 & 128.58 & 7.60   &   0.70 \\
    & \textbf{BFWS}             & 151.09 & 0.60   &   0.27 & 152.61 & 0.66   &   0.20 & 152.71 & 1.13   &   0.21 & 153.02 & 2.81   &   0.30 \\
    & \textbf{Regression-plan}  & -      & -      &   1.00 & -      & -      &   1.00 & -      & -      &   1.00 & -      & -      &   1.00 \\
    \midrule
    \multirow{2}{*}{\STAB{\rotatebox[origin=c]{90}{learn}}} & \textbf{Relational policy~\cite{relationalpolicy}} & -      & -      &   1.00 & -      & -      &   1.00 & -      & -      &   1.00 & -      & -      &   1.00 \\
    & \textbf{PLOI~\cite{ploi}}             & 182.31 & \textbf{0.20}   &   \textbf{0.00} & 169.20 & \textbf{0.24}   &   \textbf{0.00} & 161.90 & \textbf{0.34}   &   \textbf{0.00} & 152.19 & \textbf{0.61}   &   \textbf{0.00} \\
    \bottomrule
    \end{tabular}
    } %
    \label{table:taskography-courier-medium}
\end{table}

\begin{table}[!h]
    \centering
    \caption{Performance of planners over the \emph{Lifted Rearrangement(k)} domains. For all metrics, lower values indicate better performance.}
    \adjustbox{max width=\linewidth}{
    \begin{tabular}{@{}llrrrrrr@{}}
    \toprule
    &  & \multicolumn{3}{c}{\textbf{Lifted Rearr(5, 5)} Tiny} & \multicolumn{3}{c}{\textbf{Lifted Rearr(5, 5)} Medium} \\
    \cmidrule{3-5} \cmidrule{6-8}
     & \textbf{Planner}          &    \textbf{Len.} &    \textbf{Time} &   \textbf{Fail} &    \textbf{Len.} &    \textbf{Time} &   \textbf{Fail} \\
    \midrule
    \multirow{6}{*}{\STAB{\rotatebox[origin=c]{90}{satisficing}}} & \textbf{FF}               & 62.86  & 3.40   &   0.47 & 61.90  & 3.04   &   0.37 \\
    & \textbf{FF-X}             & 67.88  & 3.48   &   0.89 & 61.78  & 2.30   &   0.72 \\
    & \textbf{FD-lama-first}    & 66.81  & 3.20   &   0.49 & 71.15  & 4.11   &   0.46 \\
    & \textbf{Cerberus-sat}     & -       & -      &   1.00 & -      & -      &   1.00 \\
    & \textbf{Cerberus-agl}     & 60.50  & 7.62   &   0.60 & 64.26  & 6.74   &   0.57 \\
    & \textbf{DecStar-agl-fb}   & 66.30  & 3.02   &   0.71 & 77.00  & 3.08   &   0.71 \\
    & \textbf{BFWS}             & \textbf{56.90}  & 0.94   &   0.41 & \textbf{55.36}  & 0.80   &   0.43 \\
    & \textbf{Regression-plan}  & -      & -      &   1.00 & -      & -      &   1.00 \\
    \midrule
    \multirow{2}{*}{\STAB{\rotatebox[origin=c]{90}{learn}}} & \textbf{Relational policy~\cite{relationalpolicy}} & -      & -      &   1.00 & -      & -      &   1.00 \\
    & \textbf{PLOI~\cite{ploi}}             & 78.68  & \textbf{0.22}   &   \textbf{0.24} & 76.62  & \textbf{0.22}   &   \textbf{0.24} \\
    \bottomrule
    \end{tabular}
    } %
    \label{table:taskography-lifted-rearrangement}
\end{table}

\begin{table}[!h]
    \centering
    \caption{Performance of planners over the \emph{Lifted Courier(n, k)} domains. For all metrics, lower values indicate better performance.}
    \adjustbox{max width=\linewidth}{
    \begin{tabular}{@{}llrrrrrr@{}}
    \toprule
    &  & \multicolumn{3}{c}{\textbf{Lifted Cour(5, 5)} Tiny} & \multicolumn{3}{c}{\textbf{Lifted Cour(5, 5)} Medium} \\
    \cmidrule{3-5} \cmidrule{6-8}
     & \textbf{Planner}          &    \textbf{Len.} &    \textbf{Time} &   \textbf{Fail} &    \textbf{Len.} &    \textbf{Time} &   \textbf{Fail} \\
    \midrule
    \multirow{6}{*}{\STAB{\rotatebox[origin=c]{90}{satisficing}}} & \textbf{FF}               & \textbf{57.74}  & 4.03   &   0.44 & 57.38  & 4.81   &   0.37 \\
    & \textbf{FF-X}             & 61.19  & 7.56   &   0.77 & 60.05  & 3.79   &   0.64 \\
    & \textbf{FD-lama-first}    & 61.13  & 3.34   &   0.56 & 63.19  & 3.31   &   0.45 \\
    & \textbf{Cerberus-sat}     & -       & -      &   1.00 & -      & -      &   1.00 \\
    & \textbf{Cerberus-agl}     & 59.19  & 7.05   &   0.77 & 59.61  & 7.45   &   0.68 \\
    & \textbf{DecStar-agl-fb}   & 58.75  & 4.46   &   0.71 & 63.93  & 3.85   &   0.68 \\
    & \textbf{BFWS}             & 61.92  & 2.30   &   0.43 & \textbf{56.14}  & 0.67   &   0.38 \\
    & \textbf{Regression-plan}  & -      & -      &   1.00 & -      & -      &   1.00 \\
    \midrule
    \multirow{2}{*}{\STAB{\rotatebox[origin=c]{90}{learn}}} & \textbf{Relational policy~\cite{relationalpolicy}} & -      & -      &   1.00 & -      & -      &   1.00 \\
    & \textbf{PLOI~\cite{ploi}}             & 71.71  & \textbf{0.26}   &   \textbf{0.26} & 69.92  & \textbf{0.46}   &   \textbf{0.30} \\
    \bottomrule
    \end{tabular}
    } %
    \label{table:taskography-lifted-courier}
\end{table}

\end{document}